\documentclass[letterpaper, 10 pt, conference]{ieeeconf}  % Comment this line out if you need a4paper

\usepackage{cite}
\usepackage{amsfonts}
\usepackage{amsmath}
\usepackage{graphicx}
\usepackage{booktabs}
\usepackage{multirow}
\usepackage{xcolor}
\usepackage{pifont}
\usepackage{wrapfig}
\usepackage{stfloats}
\usepackage{url} % for supplementary-material link

\newtheorem{theorem}{Theorem}

\newcommand{\R}{\mathbb{R}}
\newcommand{\nx}{{n_x}}
\newcommand{\nuu}{{n_u}}
\newcommand{\notationfix}[1]{{\color{black}#1}} % Change blue here to recolor notation fixes.
\newcommand{\highvaluefix}[1]{{\color{black}#1}} % Change magenta here to recolor high-value notation fixes.
\newcommand{\logicfix}[1]{{\color{black}#1}} % Change orange here to recolor prose/logic fixes.

\providecommand{\MethodOurs}{KCL}
\providecommand{\BaseOne}{K-VIL}
\providecommand{\BaseTwo}{MBIRL}
\providecommand{\BaseThree}{ReKep}

\providecommand{\NA}{N/A}

\IEEEoverridecommandlockouts                              % This command is only needed if 
\title{\LARGE \bf
Learning Keypoint Constraints from Visual Demonstrations \\ for Safe Motion Planning
}

\author{
  Zheng Qiu\textsuperscript{*}, Zhouyu Zhang\textsuperscript{*}, Ninghan Zhong, and Glen Chou
}

\begin{document}

\maketitle
\thispagestyle{empty}
\pagestyle{empty}

%%%%%%%%%%%%%%%%%%%%%%%%%%%%%%%%%%%%%%%%%%%%%%%%%%%%%%%%%%%%%%%%%%%%%%%%%%%%%%%%
\begin{abstract}

We present Keypoint Constraint Learning (KCL), a framework for learning task-relevant motion constraints from visual demonstrations and for planning safely with the extracted constraints. To make constraint learning tractable from RGB-D observations, we represent demonstrations using sparse 3D object-centric keypoints and semantic mask centroids extracted through segmentation, visual feature clustering, and point tracking. In this reduced geometric space, we learn keypoint dynamics and infer distance- and orientation-based constraints that make the expert trajectories feasible and approximately locally optimal by enforcing KKT optimality conditions in the reduced state space. We then plan with the learned constraints under uncertainty, using robustness margins derived from accumulated prediction error, and provide a theoretical analysis of how keypoint perception error and learned dynamics error affect the recovered constraints. We evaluate KCL on visual demonstrations from navigation to rigid/deformable manipulation, showing that it recovers multi-stage geometric constraints and outperforms baselines in downstream safe planning.

\end{abstract}

\begin{figure*}[!b]
    \vspace{-12pt}\centering
    \includegraphics[width=\linewidth]{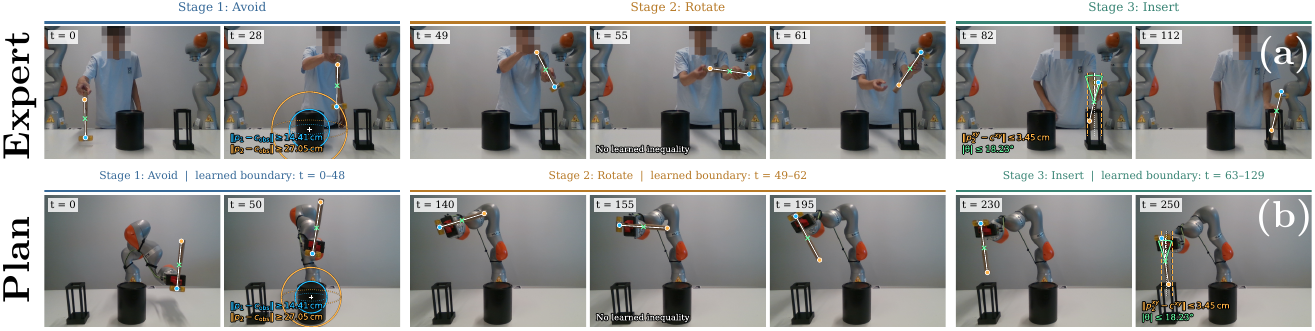}\vspace{-5pt}
    \caption{\looseness-1{\textbf{(a)} Keypoint constraints learned from a video of a human reorienting and inserting a hammer, a multi-stage task where low-level distance/orientation constraints between objects are learned together with the task segmentation (Stages 1-3). \textbf{(b)} KCL plans a safe trajectory conforming with the constraints.}\vspace{-10pt}}
    \label{fig:hammer}
\end{figure*}

\vspace{-4pt}
\section{Introduction}
\label{sec:introduction}
\vspace{-2pt}
\looseness-1Robotic behaviors depend on objectives and safety constraints that are often difficult for end users to specify. 
Constrained inverse optimal control (CIOC) methods can learn constraints and objectives from locally optimal human demonstrations \cite{ral22}, enabling safe downstream planning and control. 
Despite success in navigation and obstacle avoidance, most CIOC methods require labeled state-action demonstrations and analytical dynamics. This limitation is particularly important when accurate state estimates or representations are unavailable, as in deformable object manipulation. Video demonstrations offer a practical alternative with rich visual data and without heavy instrumentation, but raw images are high-dimensional and difficult to scale for control. Moreover, CIOC requires generalizable constraints and objectives that are difficult to define in pixel space. \logicfix{This creates a gap between vision and CIOC, motivating a visual representation that is structured enough for constraint learning.}

\looseness-1To close this gap, we use keypoints as an intermediate representation between images and CIOC. Keypoints provide a low-dimensional object-centric scene description that preserves constraint-relevant environment geometry, unlike unstructured latent representations~\cite{manuelli2019kpam}. \logicfix{Their physical interpretability simplifies nominal keypoint-space objectives, such as minimizing displacement, which regularize CIOC while allowing constraints to be inferred solely from safe expert demonstrations.} 
We propose KCL, a framework for \underline{K}eypoint \underline{C}onstraint \underline{L}earning from visual demonstrations. KCL extracts and tracks keypoints and semantic masks from paired RGB-D image-control demonstrations, and lifts them to the 3D world frame via depth maps. 
We represent the state using keypoints and mask centroids, learn dynamics in this visual space, and infer constraints from these trajectories. Assuming demonstrations noisily optimize a task objective while satisfying constraints, we enforce the Karush-Kuhn-Tucker (KKT) conditions to recover consistent constraints. 
% We then represent the system state using both keypoints and mask centroids, learn dynamics in this visual state space, and formulate constraint inference as an inverse problem over the resulting trajectories using the learned keypoint dynamics. Building on the observation that expert demonstrations should approximately optimize some task objective while obeying constraints, we enforce the Karush-Kuhn-Tucker (KKT) conditions to learn constraints that reflect the optimality of the observed demonstrations.
We analyze learning error from imperfect keypoints and plan robustly using margins derived from accumulated prediction error. 
Our contributions are: (i) a keypoint--mask representation for visual
constraint learning; (ii) KKT-based inference of geometric constraints with
learned dynamics and sensitivity bounds; (iii) a robust feedback planner; and
(iv) simulation and hardware validation on navigation and multi-stage
rigid/deformable manipulation, outperforming baselines in safety and success.
% The contributions of this work are threefold. First, we formulate constraint learning from visual demonstrations using keypoint and semantic-mask trajectories as the inference representation.%, with particular motivation from deformable-object manipulation. 
% Second, building on existing keypoint extraction and tracking methods, we develop a KKT-inspired framework for inferring task-relevant geometric constraints in the resulting visual state space. Third, we implement an optimization-based controller guided by the learned keypoint constraints and validate the full pipeline on constrained robotic demonstration tasks, including deformable-object scenarios.
%===============================================================================

\vspace{-5pt}
\section{Related Work}\label{sec:related_work}
\vspace{-2pt}

\noindent\textbf{Keypoints for Robotics. }
Keypoints are widely used as compact interfaces between perception and control ~\cite{kulkarni2019unsupervised}, enabling semantic 3D goal specification via geometric costs and constraints~\cite{manuelli2019kpam}, especially for deformable objects ~\cite{lips2024cloth} and multi-stage manipulation~\cite{huang2024rekep}.
To power keypoint-based control, many methods use VLMs to propose keypoints or spatial constraints from language and vision~\cite{kumar2026open,pan2025omnimanip,lin2025keypoint}. In contrast, we infer constraints from visual demonstrations via inverse optimality, without a retrieval database, language specification, or VLM-proposed test-time constraints, reducing hallucination and prompt errors while grounding constraints in expert behavior. VLM-free alternatives infer geometric equality constraints from demonstration variance~\cite{gao2023k}, but require manual skill and keypoint annotation~\cite{lin2025constraint}, or learn keypoint-space costs for goal reaching~\cite{mbirl} without considering constraints. Our method instead learns interpretable keypoint-mask inequality constraints for safety and task requirements without VLMs or manual keypoint annotations.

\noindent\textbf{Constraint Learning. }
 We build on CIOC, which infers inequality constraints from demonstrations using KKT optimality conditions~\cite{toussaint,ral20,corl,corl20,wafr}. CIOC is data-efficient because it exploits local optimality~\cite{ral20}, but assumes known differentiable dynamics, predefined state representations, and full state-control observations~\cite{ral20,corl20,ral22}. Other geometric constraint-learning methods require fiducial markers or focus on equality constraints~\cite{shah,liu2019learning}. Thus, existing CIOC methods are difficult to apply to visual settings without compact state representations, such as for deformable objects.
Inverse constrained RL (ICRL) is an approach for inferring constraints from demonstrations in constrained MDPs~\cite{malik2021inverse}. Unlike CIOC, it usually assumes simulator access rather than closed-form dynamics and can handle unknown dynamics~\cite{jang2023inverse}. However, existing ICRL methods still assume known, fully-observed states. We instead learn from vision and provide bounds on satisfying learned constraints under imperfect keypoint detection and dynamics.%The closest work, \cite{}, learns neural dynamics from demonstrations and uses KKT-inspired losses to learn neural constraints, but also assumes perfect state observations and does not consider visual demonstrations.

\section{Preliminaries and Problem Statement}
\label{sec:preliminaries}

% We build on two lines of work: keypoint visual representations and KKT-based constraint learning. %Keypoints compactly represent visual demonstrations as low-dimensional geometric trajectories, while KKT-based constraint learning \cite{chou2020learning,chou2022gaussian} infers unknown constraints by treating expert demonstrations as solutions of an underlying constrained optimal control problem.
% In this seection, we introduce two lines of prior work that our method builds on: keypoint-based visual representations and KKT-based constraint learning from demonstrations. Keypoints provide a compact geometric abstraction of image observations, allowing visual demonstrations to be represented as low-dimensional trajectories. KKT-based constraint learning, closely related to inverse optimal control, provides a principled way to infer unknown constraints by assuming that expert demonstrations are approximately optimal for an underlying constrained optimization problem.
% \textcolor{red}{TODO: cite related paper here.}

\noindent\textbf{Keypoints. }
% \label{subsec:prelim_keypoint}
Keypoints \cite{manuelli2019kpam} are sparse geometric features representing meaningful scene locations. They reduce raw pixels to task-relevant spatial information, such as object locations, contact regions, and relative distances. We also use semantic object masks as coarse object-level descriptors. Let $v_t \in \R^{H\times W\times 4}$ be an RGB-D observation. From each frame $v_t$, we extract $N$ keypoints $p_{i,t} \in \R^3$, where each $p_{i,t}$ denotes a world-frame 3D position, and given $M$ masks extracted from $v_t$, we denote their world-frame 3D centroids as $c_{j,t}\in \R^3$. \notationfix{We represent the keypoints $K_t \in \R^{3N}$ and centroids $C_t \in \R^{3M}$ as stacked vectors}
\begin{equation}\label{eq:keypoints}
    \notationfix{K_t:=[p_{1,t}^\top,\cdots,p_{N,t}^\top]^\top,\quad
    C_t:=[c_{1,t}^\top\cdots c_{M,t}^\top]^\top.}
\end{equation}
% \begin{equation}\label{eq:centroids}
%     C_t=\{c_{1,t},\ldots,c_{M,t}\}\in\R^{3M},
%     \quad c_{j,t}\in\mathbb{R}^{3}.
% \end{equation}

% Keypoints \cite{manuelli2019kpam, kulkarni2019unsupervised, newell2016stacked, suwajanakorn2018discovery} are sparse geometric features that represent visually or semantically meaningful locations in a scene. Compared with raw pixels, keypoints provide a much lower-dimensional representation while preserving task-relevant spatial information, such as object locations, contact regions, and relative distances. Let $v_t \in \R^{H\times W\times 4}$ denote an RGB-D visual observation at time $t$. A keypoint representation associates each $H\times W$-pixel frame $v_t$ with a set of $N$ points
% \begin{equation}\label{eq:keypoints}
%     K_t = \{p_{1,t}, \ldots, p_{N,t}\} \in \R^{3N},
%     \qquad p_{i,t} \in \mathbb{R}^{3},
% \end{equation}
%where each $p_{i,t}$ denotes the world-frame 3D position% of a tracked keypoint in the scene. 
%When only image-space detections are available, the corresponding 3D positions can be recovered using depth measurements and camera calibration. 
% \begin{equation}\label{eq:centroids}
%     C_t = \{c_{1,t}, \ldots, c_{M,t}\} \in \R^{3M},
%     \qquad c_{j,t} \in \mathbb{R}^{3}.
% \end{equation}
\notationfix{The stacked vector $P_t := [K_t^\top,C_t^\top]^\top\in\R^{3(N+M)}$ forms a reduced-order geometric state capturing object motion and structure, useful when analytical states are unavailable but visual tracks can capture task-relevant features.} %shape and motion.
% The pair $P_t := (K_t, C_t)$ defines a reduced-order geometric state that captures both local keypoint motion and object-level spatial structure. A visual demonstration can therefore be converted into a keypoint trajectory $\{(K_t, C_t)\}_{t=0}^{T}$. This representation is particularly useful for deformable-object manipulation, where a full analytic state may be difficult to define, but visually tracked points can still describe task-relevant shape and motion.

\noindent\textbf{KKT Constraint Learning. }
% \label{subsec:prelim_constraint_learning} 
% We next review the standard KKT-based view of constraint learning from demonstrations~\cite{chou2020learning,chou2022gaussian}. 
Consider a discrete-time trajectory $\xi := (x_0, \ldots, x_T, u_0, \ldots, u_{T-1})$, 
where $x_t \in \R^\nx$ is the state and $u_t \in \R^\nuu$ the control input. 
We assume each expert demonstration is an approximately locally optimal solution of a ``forward" constrained optimal control problem: 
\begin{equation}\label{eq:prelim_forward_problem}
\textstyle\min_{\xi}\; J(\xi)\quad \mathrm{s.t.}\quad h(\xi)=0,\ \ g_k(\xi)\le 0,\ \ g_\theta(\xi)\le 0.
\end{equation}
% \begin{equation}
% \begin{aligned}
%     \min_{\xi} \quad & J(\xi) \\
%     \mathrm{s.t.} \quad 
%     & h(\xi) = 0, \qquad g_k(\xi) \leq 0, \qquad g_\theta(\xi) \leq 0.
% \end{aligned}
% \label{eq:prelim_forward_problem}
% \end{equation}
$J$ is the demonstrator's cost, $h(\xi)=0$ are known equality constraints (e.g., dynamics, start/goal), and $g_k(\xi) \leq 0$ are known inequality constraints, e.g., joint limits. We denote the unknown constraint family to be learned as $g_\theta(\xi) \leq 0$, parameterized by $\theta$. If a demonstration $\xi^q$ is locally optimal for \eqref{eq:prelim_forward_problem}, then under standard constraint qualifications \cite{boyd2004convex} there exist Lagrange multipliers $\nu^q$, $\lambda_k^q$, and $\lambda_\theta^q$ such that the KKT conditions \eqref{eq:prelim_kkt}, denoted $\textrm{KKT}(\xi^q;\theta,\lambda_k^q,\lambda_\theta^q,\nu^q)$, hold:
\begin{subequations}
\label{eq:prelim_kkt}
\begin{align}
&h(\xi^q)=0,\quad g_k(\xi^q)\leq 0,\quad
g_\theta(\xi^q)\leq 0,
\label{eq:prelim_kkt_primal} \\
&\lambda_k^q\geq 0,\quad \lambda_\theta^q\geq 0,
\label{eq:prelim_kkt_dual}
\\
&\lambda_k^q\odot g_k(\xi^q)=0,\quad
  \lambda_\theta^q\odot g_\theta(\xi^q)=0,
\label{eq:prelim_kkt_comp} \\
&\begin{aligned}
\nabla_{\xi}J(\xi^q)
&+(\lambda_k^q)^\top\nabla_{\xi}g_k(\xi^q)\\
&+(\lambda_\theta^q)^\top\nabla_{\xi}g_\theta(\xi^q)
 +(\nu^q)^\top\nabla_{\xi}h(\xi^q)=0
\end{aligned}
\label{eq:prelim_kkt_stationarity}
\end{align}
\end{subequations}

\noindent Given demonstrations $\{\xi^q\}_{q=1}^{D}$, the learner solves the ``inverse" problem by finding constraints and multipliers that make all demonstrations satisfy the first-order necessary conditions for local optimality via \eqref{eq:prelim_kkt} \cite{ral20}:
\begin{equation}
\begin{aligned}
    \text{find} \quad 
    &\theta,\ \{\lambda_k^q,\lambda_\theta^q,\nu^q\}_{q=1}^{D}\\
    \mathrm{s.t.} \quad 
    &\textrm{KKT}(\xi^q;\theta,\lambda_k^q,\lambda_\theta^q,\nu^q),
    \quad q=1,\ldots,D .
\end{aligned}
\label{eq:prelim_inverse_kkt}
\end{equation}
% When the demonstrations are only approximately locally optimal, stationarity \eqref{eq:prelim_kkt_stationarity} can be relaxed and penalized as residuals. The inverse constraint learning problem \eqref{eq:prelim_inverse_kkt} aims to find constraint parameters which make the observed trajectories feasible and approximately locally optimal \cite{chou2020learning}.% In this paper, we instantiate this inverse optimization framework using visual information via keypoints. %Instead of assuming access to full physical states, we use keypoint and semantic-mask trajectories as the state representation. The unknown constraints are then defined over geometric relations in this representation, and their parameters are learned by enforcing KKT consistency of the demonstrated keypoint--mask trajectories.

\noindent\textbf{Problem Statement. }
\looseness-1We are given $D$ demonstrations
$\mathcal{D}=\{(v^q_{0:T_q},u^q_{0:T_q-1})\}_{q=1}^D$, 
where $v^q_t \in \R^{H\times W\times 4}$ is the RGB-D image at time $t$ and $u^q_t\in\mathbb{R}^{n_u}$ is the control input. 
% From each visual sequence, we extract tracked keypoints $K^q_t$ and semantic mask centroids $C^q_t$, and form the compact visual state
% \[
% P^q_t := 
% \begin{bmatrix}
% K^q_t \\
% C^q_t
% \end{bmatrix}.
% \]
% The resulting state-control trajectory is
% \[
% z^q := \{P^q_0,\ldots,P^q_{T_q},u^q_0,\ldots,u^q_{T_q-1}\}.
% \]
% We assume that the demonstrations are approximately locally-optimal solutions of a constrained optimal control problem in a reduced-order keypoint--mask state space extracted from the visual demonstrations. The task objective is additive across time and parameterized by known trajectory features $\phi(\cdot)$ with unknown weights $\theta_J$, i.e., $J_{\theta_J}(z)=\sum_{t=0}^T\theta_J^\top \phi(P_t)$, 
\notationfix{We assume the demonstrations are (potentially noisily) locally-optimal solutions of a constrained optimal control problem in the reduced keypoint--mask state space extracted from the images. For demonstration $q$, the task objective is
$J_{\theta_J}(P^q_{0:T_q}) := \sum_{t=0}^{T_q-1}\sum_{i=1}^{N}
\theta_{J,i}\|p^q_{i,t+1}-p^q_{i,t}\|_2^2$, with unknown weights $\theta_J$.}
% where $z=\{P_0,\ldots,P_T,u_0,\ldots,u_{T-1}\}$ denotes the keypoint--mask state-control trajectory.
\notationfix{The unknown task constraints belong to a differentiable family $g_{\theta}(P_t,u_t)\le 0$. Within each task stage, the parameter vector $\theta$ is shared across time; for multi-stage tasks, each stage has its own parameter vector, and the stage boundaries may also be inferred. Our goal is to 1) infer $\theta_J$ and the stage-wise constraint parameters $\theta$ such that the demonstrations in $\mathcal D$ are feasible and locally optimal under $J_{\theta_J}$, and 2) plan safely using the learned parameters.}

% \paragraph{Problem Statement}%\label{sec:problem_statement}

% We are given $D$ paired visual and control demonstrations. The visual component are given as RGB images; possibly depth as well. We assume knowledge of a cost function parameterization in the keypoint space, where $J(\xi) = \theta_J^\top \phi(\xi)$ and $\theta_J$ are unknown and to be learned. Constraint parameterization, time-varying.
% Aim to learn cost and constraint parameters.
    
%===============================================================================
\section{Methodology}
\label{sec:methodology}

% Our framework learns constraints from visual demonstrations by converting image observations into a reduced-order geometric state and then applying optimization-based inverse reasoning in that state space. The overall pipeline consists of three components. First, we extract and track task-relevant keypoints and semantic object masks from demonstration videos, producing keypoint--mask trajectories. Second, we formulate constraint learning over this visual state representation by defining candidate geometric constraints between keypoints, mask centroids, and controls, and infer their parameters using KKT-inspired consistency conditions. Third, we use the learned keypoint-space dynamics and constraints to plan new trajectories through constrained optimization. This section describes each component in turn.
% Our framework consists of three components, overviewed in the following. First, we extract and track task-relevant keypoints and semantic object masks from demonstration videos, producing keypoint--mask trajectories (Sec. \ref{subsec:method_visual_representation}). Second, we formulate constraint learning over this visual state representation by defining candidate geometric constraints between keypoints, mask centroids, and controls, and infer their parameters using KKT-inspired consistency conditions. Third, we use the learned keypoint-space dynamics and constraints to plan new trajectories through constrained optimization (Sec. \ref{subsec:method_constraint_learning}).
Our method first extracts keypoints and semantic masks from demonstrations (Sec. \ref{subsec:method_visual_representation}). Next, we learn geometric constraints over keypoints and centroids via the KKT conditions (Sec. \ref{subsec:method_constraint_learning}). Finally, we plan robustly with the learned keypoint dynamics and constraints (Sec. \ref{subsec:method_planning}).

\subsection{Visual Demonstrations to Keypoint--Mask Trajectories}
\label{subsec:method_visual_representation}

% Given visual demonstrations $\mathcal{D}$, we convert each raw sequence into a low-dimensional trajectory for dynamics learning and constraint inference. At time $t$, the reduced geometric state is $P_t := (K_t, C_t)$, where $K_t$ and $C_t$ denote the $N$ keypoints and $M$ semantic mask centroids from \eqref{eq:keypoints} and \eqref{eq:centroids}, capturing local motion and object-level structure. 
\looseness-1Given $\mathcal{D}$, we convert each raw demonstration into a low-dimensional trajectory for constraint inference. \notationfix{At time $t$, the geometric state is the stacked vector $P_t=[K_t^\top,C_t^\top]^\top$ defined in Sec.~\ref{sec:preliminaries}.} %, capturing local motion and object-level structure.
In simulation, masks and initial keypoints are directly available; in the real world, they are not. For real-world demonstrations, we use SAM ~\cite{kirillov2023segment} to extract initial masks 
$\mathcal{M}_0=\{m_0^k\}_{k=1}^{M}$, $m_0^k\in\R^{H\times W}$, from $v_0$. 
These masks restrict keypoints to foreground regions and define semantic pixel-level centroids $\hat C_0 = \{\hat c_0^k\}_{k=1}^{M}$. 
\notationfix{Within each mask, we use DINOv2 features~\cite{oquab2023dinov2} to select keypoints. We cluster the mask-restricted dense DINOv2 features with $k$-means and, for each cluster, select the mask pixel whose feature is closest to the cluster center. The selected pixels form the initial image-space keypoints $\hat K_0$.}
% To obtain $P_t$ for real-world visual demonstrations, we use the Segment Anything Model (SAM) \cite{kirillov2023segment} to obtain object masks from the image observations. These masks restrict keypoint proposal to meaningful foreground regions and provide the semantic regions whose centroids form $C_t$. Within each selected mask, we use DINOv2 visual features ~\cite{oquab2023dinov2} to identify candidate keypoints. Concretely, dense image features are extracted from the frame, grouped according to the corresponding object masks, and clustered to propose a small number of representative points for each region. This produces an initial set of task-relevant keypoints in the first frame of each demonstration.
% After initialization, we track these keypoints through the video using TAPNext++~\cite{jung2026tapnext}, which estimates queried pixel trajectories and visibility over time. Depth observations and camera calibration lift the tracked image-space keypoints $\hat K_t$ to world-frame 3D coordinates $K_t$ via back-projection. The same projection computes 3D mask centroids from segmented regions. For the $q$th demonstration, this yields a reduced geometric state-control trajectory
We then track keypoints over time with TAPNext++~\cite{jung2026tapnext}, which estimates pixel trajectories and visibility. Depth and camera calibration lift tracked image keypoints $\hat K_t$ to world-frame 3D coordinates $K_t$ via back-projection, and similarly compute 3D mask centroids. For the $q$th demonstration, this yields a reduced state-control trajectory $z^q$ for keypoint constraint learning:
% After initialization, we track the selected keypoints through the video using TAPNext++ \cite{jung2026tapnext}, a tracking-any-point model designed to estimate the trajectory of queried pixels across frames of a video, outputting image-space keypoint locations and visibility estimates over time. We use depth observations and camera calibration to lift the tracked image-space keypoints $\hat K_t$ to 3D coordinates by back-projecting each pixel using its depth value and transforming the resulting point into the world frame, yielding the world-frame 3D keypoints $K_t$. We apply the same projection procedure to compute 3D mask centroids from the corresponding segmented regions. The result of this stage is a geometric reduced state-control trajectory
\begin{equation}\label{eq:keypoint_demo}
\begin{aligned}
z^q &:=\notationfix{{\{(K_t^q,C_t^q,u_t^q)\}}_{t=0}^{T_q-1}
       \cup\{(K_{T_q}^q,C_{T_q}^q)\}}\\
    &=\notationfix{\{P^q_0,\ldots,P^q_{T_q},u^q_0,\ldots,u^q_{T_q-1}\}}.
\end{aligned}
\end{equation}
% where $u_t$ is the control input when it is available. This trajectory
% and serves as the state-control data for keypoint constraint learning in Sec. \ref{subsec:method_constraint_learning}.% Constraints over keypoint-mask  are differentiable with respect to this state and compatible with the KKT-based formulation below.
% This keypoint--mask representation converts high-dimensional RGB-D images into a low-dimensional state for dynamics learning and provides a geometric space for candidate constraints. Constraints can be defined over keypoint-mask distances or keypoint-pair orientations, all differentiable with respect to the visual state and compatible with the KKT-based constraint-learning formulation below.
% This keypoint--mask representation serves two purposes in our framework. First, it converts high-dimensional visual observations into a low-dimensional state suitable for learning dynamics. Second, it provides a geometric space in which candidate constraints can be expressed. For example, constraints can be defined over keypoint--mask distances, keypoint pair orientations, or final-time goal relations. These quantities are differentiable with respect to the visual state, allowing them to be incorporated into the KKT-based constraint-learning formulation introduced below.

\subsection{Keypoint Constraint Learning}
\label{subsec:method_constraint_learning}

Given the keypoint--mask trajectories extracted from visual demonstrations, we formulate constraint learning as an inverse optimization problem in the learned visual state space. 
% Let
% \begin{equation}
%     P_t :=
%     \begin{bmatrix}
%         K_t \\
%         C_t
%     \end{bmatrix}
%     \in \mathbb{R}^{(N+M)\times 3},
% \end{equation}
% where $K_t=\{p_{1,t},\ldots,p_{N,t}\}$ denotes the tracked 3D keypoints and $C_t=\{c_{1,t},\ldots,c_{M,t}\}$ denotes the semantic mask centroids. We use $P_t$ as the compact visual state. 
Following \eqref{eq:keypoint_demo}, for the $q$-th demonstration, the corresponding visual state-control trajectory is denoted by $z^q$.

\noindent\textbf{Learning keypoint dynamics. }
% The KKT-based inverse problem requires equality constraints that describe feasible state transitions. Since the full physical state and analytic dynamics may be unavailable, especially for deformable objects, we learn a keypoint-space dynamics model from demonstration data. Specifically, we train a neural dynamics model $f_\psi:\R^{3}\times \R^\nuu\rightarrow \R^3$ that predicts a per-point displacement conditioned on the current point position and the control input
\looseness-1The KKT inverse problem requires equality constraints describing feasible state transitions. Since analytic dynamics may be unavailable, especially for deformable objects, we learn keypoint-space dynamics from demonstrations. \notationfix{A neural network is used to learn the state-transition map $f_\psi:\R^{3(N+M)}\times\R^\nuu\rightarrow\R^{3(N+M)}$,}
\notationfix{$P_{t+1}=f_\psi(P_t,u_t):=P_t+\alpha_\psi(P_t,u_t)\odot\Delta^{\mathrm{ctrl}}_\psi(P_t,u_t)+\Delta^{\mathrm{free}}_\psi(P_t)$.}
% \begin{equation}
%     \resizebox{\columnwidth}{!}{$\displaystyle
%     P_{t+1} = f_\psi(P_t,u_t) :=P_t + \alpha_\psi(P_t,u_t)\odot \Delta^{\mathrm{ctrl}}_\psi(P_t,u_t)
%     +
%     \Delta^{\mathrm{free}}_\psi(P_t),%\Delta_\psi(P_t,u_t), 
%     $}
%     \label{eq:keypoint_dynamics}
% \end{equation}
% where $\Delta_\psi$ is an MLP %. The model :
% \begin{equation}
%     f_\psi(P_t,u_t)
%     =
%     P_t + \Delta_\psi(P_t,u_t).
% \end{equation}
with a control-dependent component $\Delta^{\mathrm{ctrl}}_\psi$ and a control-independent component $\Delta^{\mathrm{free}}_\psi$, 
% \begin{equation}
%     \Delta_\psi(P_t,u_t)
%     =
%     \alpha_\psi(P_t,u_t)\odot \Delta^{\mathrm{ctrl}}_\psi(P_t,u_t)
%     +
%     \Delta^{\mathrm{free}}_\psi(P_t),
%     \label{eq:keypoint_dynamics_residual}
% \end{equation}
where $\alpha_\psi$ is a learned gate. This factorization enables capturing both passive visual-state evolution and control-driven motion. 
The dynamics are trained via one-step prediction over demonstration data $\{z^q\}_{q=1}^D$:
\begin{equation}
    \min_{\psi}
    \textstyle\sum_{q=1}^{D}
    \sum_{t=0}^{T_q-1}
    \left\|
        f_\psi(P^q_t,u^q_t) - P^q_{t+1}
    \right\|_1.
    \label{eq:keypoint_dynamics_training}
\end{equation}
% After training, 
$f_\psi$ is treated as an equality constraint in the forward and inverse optimization problems. Its Jacobians with respect to the visual state and control, $A_t := \textstyle\frac{\partial f_\psi}{\partial P_t}(P_t,u_t)$ and $B_t := \frac{\partial f_\psi}{\partial u_t}(P_t,u_t)$, 
are computed by automatic differentiation and used in the stationarity equations of the KKT system \eqref{eq:keypoint_kkt_stationarity}.

\noindent\textbf{Keypoint constraints. }
\looseness-1We set candidate constraint parameterizations as geometric relations on $P_t$. 

\noindent\textit{Distance Constraints: }
\notationfix{We consider the pairwise keypoint-mask distance constraint. For distance candidate $i$, corresponding to keypoint $p_{a_i,t}$ and mask centroid $c_{b_i,t}$, we define}
\begin{equation}
    \notationfix{g^{\mathrm{dist}}_{i,t}(P_t;\rho_i)
    :=
    \rho_i
    -
    \|p_{a_i,t}-c_{b_i,t}\|_2^2
    \leq 0,}
    \label{eq:keypoint_mask_distance_constraint}
\end{equation}
\notationfix{where $\rho_i\geq 0$ is the learned squared distance margin, shared across time within the current task stage.} This family is used to represent obstacle avoidance, object separation, and task-relevant spatial relations.% in both rigid-body and deformable-object demonstrations.
% The primary constraint family used in our experiments is the pairwise keypoint--mask distance constraint. For each keypoint $p_{i,t}$ and mask centroid $c_{j,t}$, we define
% \begin{equation}
%     g^{\mathrm{dist}}_{i,j,t}(P_t;\rho_{i,j,t})
%     :=
%     \rho_{i,j,t}
%     -
%     \|p_{i,t}-c_{j,t}\|_2^2
%     \leq 0,
%     \label{eq:keypoint_mask_distance_constraint}
% \end{equation}
% where $\rho_{i,j,t}\geq 0$ is the learned squared distance margin at time $t$. Equivalently, $\rho_{i,j,t}=\theta_{i,j,t}^2$, where $\theta_{i,j,t}$ is the minimum distance between the corresponding keypoint and semantic region. This candidate family is used to represent obstacle avoidance, object separation, and task-relevant spatial relations in both rigid-body and deformable-object demonstrations.

\noindent\textit{Orientation Constraints: }
\notationfix{We also consider keypoint-pair orientation constraints. For orientation candidate $j$, corresponding to a keypoint pair $(a_j,b_j)$ and a fixed world-frame unit axis $\hat e_j\in\{\hat e_x,\hat e_y,\hat e_z\}$, define
$v_{j,t}:=p_{a_j,t}-p_{b_j,t}$ and
$\phi_j(P_t):=((v_{j,t}^{\top}\hat e_j)/\|v_{j,t}\|_2)^2$. We then impose}
\begin{equation}
\notationfix{
g^{\mathrm{ori}}_{j,t}(P_t;\sigma_j)
:=\sigma_j-\phi_j(P_t)\leq 0,
}
    \label{eq:keypoint_orientation_constraint}
\end{equation}
\notationfix{where $\sigma_j\in[0,1]$ is the learned axis-alignment threshold, shared across time within the current task stage.} 
Let $g_\theta(z)$ denote all unknown candidate constraints stacked over time, with parameters $\theta$ collecting the distance and orientation thresholds. 
% , and any task- or stage-specific parameters. 
We also include known control limit constraints, represented component-wise with known $\bar u_{\ell}^{2}$ as
\begin{equation}
    g^{u}_{\ell,t}(u_t)
    :=
    u_{t,\ell}^{2} - \bar u_{\ell}^{2}
    \leq 0,
    \qquad \ell=1,\ldots,\nuu.
    \label{eq:control_constraint}
\end{equation}

\noindent\textbf{Keypoint forward problem. }
We assume that each demonstration is approximately locally optimal for a constrained optimal control problem in the learned visual state space:
\begin{equation}
\begin{aligned}
\min_{\substack{P_{0:T},\\u_{0:T-1}}}\quad
&J(P_{0:T}):=\textstyle\sum_{t=0}^{T-1}\sum_{i=1}^{N}\theta_{J,i}\cdot\|p_{i,t+1}-p_{i,t}\|_2^2\hspace{-7pt}\\
\mathrm{s.t.}\quad
&P_{t+1}=f_\psi(P_t,u_t),\quad t=0,\ldots,T-1,\\
&P_0=P_{\mathrm{start}},
\quad P_T=P_{\mathrm{goal}},\\
&g_u(u_{0:T-1})\leq0,
\quad g_\theta(P_{0:T})\leq0.
\end{aligned}
\label{eq:keypoint_forward_problem}
\end{equation}
Here, $P_{\mathrm{start}}$ and $P_{\mathrm{goal}}$ encode the task-specified start and goal conditions. \logicfix{We use a weighted keypoint-displacement objective $J(P_{0:T})$ in \eqref{eq:keypoint_forward_problem}, although other task costs can be used. The objective need not fully explain expert intent; it provides a nominal optimality model for inferring constraints. Because keypoints are geometric, penalizing their displacement favors short, low-motion trajectories typical of expert demonstrations.}
% although other task costs can be incorporated. The objective is not intended to fully explain expert intent; rather, it provides a nominal optimality model under which the hidden constraints can be inferred from the demonstrations. The choice in \eqref{eq:keypoint_smoothness_cost} is natural in our setting because the keypoints carry geometric meaning: penalizing their displacement encourages short, smooth object motions natural of an expert demonstrator.

\noindent\textbf{KKT-based inverse problem. }
Following KKT-based constraint learning (Sec. \ref{sec:preliminaries}) \cite{ral20}, we search for constraint parameters $\theta$ and multipliers that make the demonstrations approximately satisfy the first-order optimality conditions of \eqref{eq:keypoint_forward_problem}. For each demonstration $z^q$, \notationfix{define the Lagrangian of the forward problem as}
\begin{equation}
\hspace{-8pt}\begin{aligned}
&\notationfix{\mathcal{L}(z^q;\theta,\theta_J,\lambda_\theta^q,\lambda_u^q,\nu^q)} :=\notationfix{J_{\theta_J}(P^q_{0:T_q})+(\lambda_\theta^q)^\top g_\theta(P^q_{0:T_q})}\hspace{-8pt}\\
&\qquad\qquad \notationfix{
+(\lambda_u^q)^\top g_u(u^q_{0:T_q-1})}\notationfix{+(\nu^q)^\top h_\psi(z^q),}
\end{aligned}
\label{eq:keypoint_lagrangian}
\end{equation}
\logicfix{where $h_\psi(z^q)=0$ stacks the learned dynamics constraints; the demonstrated boundary values are treated as fixed data and are not relaxed.} The KKT conditions are
\begin{subequations}
\label{eq:keypoint_kkt}
\begin{align}
    &\begin{aligned}
    h_\psi(z^q)&=s^q, & \|s^q\|_\infty&\le \epsilon_{\rm dyn},\\
    \notationfix{g_\theta(P^q_{0:T_q})}&\leq0, & \notationfix{g_u(u^q_{0:T_q-1})}&\leq0
    \end{aligned}
    \label{eq:keypoint_kkt_primal}
    \\
    &\begin{aligned}
    \lambda_\theta^q&\geq0, & \lambda_u^q&\geq0,\\
    \notationfix{\lambda_\theta^q\odot g_\theta(P^q_{0:T_q})}&=0, &
    \notationfix{\lambda_u^q\odot g_u(u^q_{0:T_q-1})}&=0
    \end{aligned}
    \label{eq:keypoint_kkt_dual_comp}
    \\
    &\nabla_{z^q}\mathcal{L}
    \notationfix{(z^q;\theta,\theta_J,\lambda_\theta^q,\lambda_u^q,\nu^q)}=0.
    \label{eq:keypoint_kkt_stationarity}
\end{align}
\end{subequations}
where we relax the dynamics equality constraints \eqref{eq:keypoint_kkt_primal} to account for learned model error. \notationfix{Here $\epsilon_{\rm dyn}\ge 0$ bounds the equality-constraint mismatch.} The inverse KKT problem thus mirrors \eqref{eq:prelim_inverse_kkt}. 
% \begin{equation}
% \begin{aligned}
%     \text{find} \quad
%     & \theta,\theta_J,\ 
%     \{\lambda_\theta^q,\lambda_u^q,\nu^q\}_{q=1}^{D} \\
%     \mathrm{s.t.} \quad
%     & \mathrm{KKT}(z^q;\theta,\theta_J,\lambda_\theta^q,\lambda_u^q,\nu^q),
%     \qquad q=1,\ldots,D.
% \end{aligned}
% \label{eq:keypoint_inverse_problem_exact}
% \end{equation}
\looseness-1\logicfix{Because keypoint tracking error, model error in $f_\psi$, and suboptimal demonstrations can make the exact KKT conditions infeasible, we solve a relaxed inverse problem} that preserves primal and dual feasibility and minimizes stationarity residuals. \notationfix{Let $\Theta$ and $\Theta_J$ denote the admissible constraint- and cost-parameter sets, respectively, and let $\beta>0$ weight the dynamics-equality slack:}
\begin{equation}
\begin{aligned}
\min_{\substack{\theta,\theta_J,\\
\{\lambda_\theta^q,\lambda_u^q,\\\nu^q,s^q\}_{q=1}^{D}}}\ 
&\notationfix{\textstyle\sum_{q=1}^{D}\left[
\bigl\|\nabla_{z^q}\mathcal{L}
(z^q;\theta,\theta_J,\lambda_\theta^q,\lambda_u^q,\nu^q)\bigr\|_1
+\beta\|s^q\|_1\right]}\\
\mathrm{s.t.}\ \
&\eqref{eq:keypoint_kkt_primal},\ \eqref{eq:keypoint_kkt_dual_comp},
\quad \theta\in\Theta,\ \theta_J\in\Theta_J,\quad q=1,\ldots,D.
\end{aligned}
\label{eq:keypoint_inverse_problem_relaxed}
\end{equation}
Complementarity can be enforced exactly using binary variables and Big-M constraints \cite{ral20} by framing \eqref{eq:keypoint_inverse_problem_relaxed} as a mixed-integer program.  
% ; in our implementation, we first attempt to solve exact KKT feasibility \eqref{eq:keypoint_inverse_problem_exact} and fall back to the relaxed stationarity formulation \eqref{eq:keypoint_inverse_problem_relaxed} when needed. 
The recovered parameters $\theta$ define the learned keypoint constraints, and the associated multipliers indicate which keypoint constraints are active along the demonstrations. 
\logicfix{For multi-stage tasks, we introduce stage-dependent candidate constraints and parameters. Stage assignments may be specified or jointly inferred with the constraint parameters, as in Sec.~\ref{subsec:exp_deformable}. Different geometric relations can therefore explain different task phases.} 
% Since the inverse problem~\eqref{eq:keypoint_inverse_problem_relaxed} includes many
% candidate keypoint--mask pairs $(i,j)$ and orientation pairs $(a,b,\hat{e})$ when
% $N$ and $M$ are large, we prune inactive constraints to
% reduce the number enforced during planning. Details of the pruning and
% sensitivity checks are in \cite[App. B]{kcl_supplement}.
% The inverse problem~\eqref{eq:keypoint_inverse_problem_relaxed} includes many
% candidate keypoint--mask pairs $(i,j)$ and orientation pairs $(a,b,\hat{e})$.
% Their observed geometric extrema alone do not identify which constraints explain
% the demonstrations, while exhaustive subset enumeration over $NM$ distance
% candidates involves $2^{NM}$ possibilities. KCL jointly infers constraint
% parameters and activity through the KKT conditions, then prunes candidates that
% are inactive throughout the demonstrations or whose boundary loosening causes
% negligible increase in the minimum stationarity residual. This retains
% constraints with explanatory value under the assumed objective and learned
% dynamics, while reducing unnecessary restrictions during planning. Details of
% the pruning and sensitivity checks are in \cite[App. B]{kcl_supplement}.
Rather than enumerating subsets of the many candidate distance and orientation
relations, KCL jointly infers their parameters and activity through the KKT
conditions. It then prunes candidates that are always inactive or insensitive
to boundary loosening, retaining explanatory constraints while avoiding
unnecessary planning restrictions; details are in \cite[App. B]{kcl_supplement}.

\noindent\textbf{Sensitivity of KCL. }
As the learning problem \eqref{eq:keypoint_inverse_problem_relaxed} is impacted by keypoint perception and dynamics errors, \logicfix{we derive bounds on the resulting error in the learned distance and orientation thresholds:} 

\begin{theorem}[Learned distance/orientation error bound]\label{thm:error_bound}
Consider the inverse problem \eqref{eq:keypoint_inverse_problem_relaxed}
\notationfix{on ideal reduced trajectories $\mathcal Z=\{z^q\}_{q=1}^D$ and perceived
approximations $\hat{\mathcal Z}=\{\hat z^q\}_{q=1}^D$, following
\eqref{eq:keypoint_demo}. Assume the constraint parameters are shared across time within the stage under consideration and
$\theta_J$ is known.} Suppose the perceived keypoints and centroids satisfy
$\|\hat p_{i,t}^q-p_{i,t}^q\|_2\le \epsilon_k$ and
$\|\hat c_{j,t}^q-c_{j,t}^q\|_2\le \epsilon_c$ for all $q,t,i,j$, and define
$r_d:=\epsilon_k+\epsilon_c$. Let $\rho\in\mathbb R^{n_d}$ and
$\sigma\in\mathbb R^{n_o}$ denote the learned distance-margin and
orientation-threshold parameters in
\eqref{eq:keypoint_mask_distance_constraint} and
\eqref{eq:keypoint_orientation_constraint}, respectively. Assume all
orientation-defining keypoint pairs satisfy
$\|p_{a,t}^q-p_{b,t}^q\|_2\ge \ell_{\min}>2\epsilon_k$, and define
$r_o:=2\epsilon_k/(\ell_{\min}-2\epsilon_k)$. 
Let $(\rho^\star,\sigma^\star)$, $(\tilde\rho,\tilde\sigma)$, and
$(\hat\rho,\hat\sigma)$ denote the recovered parameters using, respectively:
exact perception and exact dynamics Jacobians; perturbed perception and exact
Jacobians at the perceived trajectory; and perturbed perception with learned
Jacobians. Assume these three problems share the same reduced active set and
that each active geometric parameter remains pinned by the same feature, i.e.,
$\rho_i^\star=\|p_{a_i,t_i}^{q_i}-c_{b_i,t_i}^{q_i}\|_2^2$,
$\tilde\rho_i=\|\hat p_{a_i,t_i}^{q_i}-\hat c_{b_i,t_i}^{q_i}\|_2^2$,
$\sigma_j^\star=\phi_j(P_{t_j}^{q_j})$, and
$\tilde\sigma_j=\phi_j(\hat P_{t_j}^{q_j})$.

\notationfix{Let $f^\star$ denote the exact keypoint-space dynamics map, and define
$\tilde A_t^q:=\frac{\partial f^\star}{\partial P}(\hat P_t^q,u_t^q)$,
$\tilde B_t^q:=\frac{\partial f^\star}{\partial u}(\hat P_t^q,u_t^q)$,
$\hat A_t^q:=\frac{\partial f_\psi}{\partial P}(\hat P_t^q,u_t^q)$, and
$\hat B_t^q:=\frac{\partial f_\psi}{\partial u}(\hat P_t^q,u_t^q)$.
Suppose the learned Jacobians satisfy
$\left\|\begin{bmatrix}
\hat A_t^q-\tilde A_t^q & \hat B_t^q-\tilde B_t^q
\end{bmatrix}\right\|_2\le \epsilon_f$ for all $q,t$.}
\highvaluefix{Let $F(y;\zeta)=0$, with
$y=(\rho,\sigma,\lambda,\nu,\ldots)$, denote the reduced active-set inverse-KKT
system. For the data $\tilde\zeta$ with perturbed perception and exact Jacobians,
define $\mathcal U_{\rm jac}$ as all data with the same perceived trajectory as
$\tilde\zeta$ and Jacobian perturbations satisfying
$\max_{q,t}\left\|\begin{bmatrix}
A_t^q-\tilde A_t^q & B_t^q-\tilde B_t^q
\end{bmatrix}\right\|_2\le\epsilon_f$. Assume that, for every
$\zeta\in\mathcal U_{\rm jac}$, the reduced KKT system has a locally unique solution
with the same active set, that
$\bar\nu:=\sup_{\zeta\in\mathcal U_{\rm jac}}\|\nu(\zeta)\|_2<\infty$, and that there
exist $\kappa_\rho,\kappa_\sigma>0$ such that
$\|\rho(\zeta)-\rho(\tilde\zeta)\|_2
\le
\kappa_\rho
\|F(y(\tilde\zeta);\zeta)-F(y(\tilde\zeta);\tilde\zeta)\|_2$
and
$\|\sigma(\zeta)-\sigma(\tilde\zeta)\|_2
\le
\kappa_\sigma
\|F(y(\tilde\zeta);\zeta)-F(y(\tilde\zeta);\tilde\zeta)\|_2$.}
Finally, let
$\tilde d_{\max}:=\max_i
\|\hat p_{a_i,t_i}^{q_i}-\hat c_{b_i,t_i}^{q_i}\|_2$. Then
\begin{equation}
\begin{aligned}
\|\hat\rho-\rho^\star\|_2
&\le\sqrt{n_d}\left(2\tilde d_{\max}r_d+r_d^2\right)
+\kappa_\rho\epsilon_f\bar\nu,\\
\|\hat\sigma-\sigma^\star\|_2
&\le\sqrt{n_o}\,r_o+\kappa_\sigma\epsilon_f\bar\nu.
\end{aligned}
\end{equation}
\end{theorem}
A proof is provided in \cite[App. A]{kcl_supplement}.

\vspace{3pt}
\subsection{Planning with Keypoint Dynamics and Constraints}
\label{subsec:method_planning}

\looseness-1Given the learned dynamics $f_\psi$ and constraint
parameters $\theta$, we plan new trajectories by solving the
forward control problem
\eqref{eq:keypoint_forward_problem} from a new start condition
$P_{\rm start}$ to a desired goal $P_{\rm goal}$. 
% \paragraph{Trajectory optimization.}
Parameterizing the plan as a sequence of reduced states and controls
$\{P_0, \ldots, P_T, u_0, \ldots, u_{T-1}\}$, \highvaluefix{let $\eta_1,\eta_2\ge 0$ be cost weights and let
$\ell(P_T)$ denote the terminal goal cost, e.g.,
$\ell(P_T)=\sum_i\|p_{i,T}-p_{i,\rm goal}\|_2^2$ for relevant keypoints.}
Using IPOPT, we solve to local optimality the following:
\begin{equation}
\begin{aligned}
\hspace{-7pt}\min_{\substack{P_{0:T},\\u_{0:T-1}}}\ 
&\highvaluefix{\textstyle\sum_{t=0}^{T-1}\|K_{t+1}-K_t\|_2^2 + \eta_1\sum_{t=0}^{T-1}\|u_t\|_2^2+\eta_2\ell(P_T)}\\
\mathrm{s.t.}\quad
&P_{t+1}=f_\psi(P_t,u_t), \quad t=0,\ldots,T-1,\\
&P_0=P_{\rm start},\\
    % & \|P_{t+1} - f_\psi(P_t,u_t)\|_2^2 \;\leq\; \delta_{\rm dyn}^2,
    % \qquad t=0,\ldots,T-1;\qquad P_0 = P_{\rm start},\\
&g_\theta(P_t)\leq0, \quad t=0,\ldots,T,\\
&g_u(u_t)\leq0, \quad t=0,\ldots,T-1.
\end{aligned}
\label{eq:planning_problem}
\end{equation}
% The dynamics equality $P_{t+1}=f_\psi(P_t,u_t)$ is relaxed to a
% soft tube $\|P_{t+1}-f_\psi(P_t,u_t)\|_2^2 \leq \delta_{\rm dyn}^2$,
% which improves numerical conditioning and tolerates small model
% inaccuracies. 
\logicfix{The learned geometric constraints $g_\theta$ are enforced as hard inequalities, so the nominal model trajectory satisfies the inferred task requirements. We next account for deviations caused by learned-model error during execution.}

% \paragraph{Robust planning.}
% Since the learned dynamics $f_\psi$ are imperfect, small one-step prediction errors can accumulate into constraint violations during execution. We therefore robustify \eqref{eq:planning_problem} using a system level synthesis (SLS) \cite{anderson2019system} tube around the nominal keypoint trajectory. Let $(\bar P_{0:T},\bar u_{0:T-1})$ denote the nominal plan.
\looseness-1Since the learned dynamics $f_\psi$ are imperfect, small one-step prediction
errors can accumulate into constraint violations during execution. We therefore
robustify \eqref{eq:planning_problem} using a system-level synthesis (SLS)
\cite{anderson2019system} tube around the nominal keypoint trajectory. Let
$(\bar P_{0:T},\bar u_{0:T-1})$ denote the nominal plan. Linearizing the learned
keypoint dynamics around this plan gives the tracking-error dynamics
 $\delta P_{t+1} = A_t \delta P_t + B_t \delta u_t + w_t$, where $A_t := \frac{\partial f_\psi}{\partial P}(\bar P_t,\bar u_t)$, $B_t := \frac{\partial f_\psi}{\partial u}(\bar P_t,\bar u_t)$, 
% where $w_t \in \mathcal W_t$ represents bounded keypoint-dynamics error estimated from held-out one-step prediction residuals following \cite{srinivasan2026safety}. SLS parameterizes the closed-loop responses from disturbances to keypoint ($\delta P$) and control deviations ($\delta u$) as
and where $w_t \in \mathcal W_t$ represents bounded keypoint-dynamics error,
estimated from held-out one-step prediction residuals following
\cite{srinivasan2026safety}. Let
$w := (w_0,\ldots,w_{T-1})$ denote the stacked disturbance sequence. SLS
parameterizes the closed-loop responses from disturbances to keypoint and
control deviations as $\delta P = \Phi_P w$ and $\delta u = \Phi_u w$, 
where $(\Phi_P,\Phi_u)$ satisfy the finite-horizon system-level dynamics constraints induced by $(A_t,B_t)$. 
For each learned distance constraint indexed by $i$ \eqref{eq:keypoint_mask_distance_constraint}, let $(a_i,b_i)$ denote the
corresponding keypoint--centroid pair. 
\highvaluefix{Define the stacked disturbance set
$\mathcal W:=\mathcal W_0\times\cdots\times\mathcal W_{T-1}$.}
From $\Phi_P$ and the disturbance sets $\{\mathcal W_t\}_{t=0}^{T-1}$, we compute
\highvaluefix{the worst-case tube radius} \eqref{eq:tube} for this relative displacement, which can be evaluated in closed-form for SLS \cite{srinivasan2026safety}:
\begin{equation}\label{eq:tube}
\begin{aligned}
    \highvaluefix{r_{i,t}}&\highvaluefix{:=
    \textstyle\sup_{w\in\mathcal W}
    \|\delta p_{a_i,t}(w)-\delta c_{b_i,t}(w)\|_2,}
\end{aligned}
\end{equation}
We then replace the distance constraint in
\eqref{eq:planning_problem} with the tightened robust constraint
\begin{equation}
    g^{\mathrm{rob}}_{i,t}(\bar P_t)
    :=
    \left(\sqrt{\rho_i} + r_{i,t}\right)^2
    -
    \|\bar p_{a_i,t} - \bar c_{b_i,t}\|_2^2
    \le 0.
\end{equation}
Thus, the nominal planned keypoint--centroid distance must exceed the learned
clearance $\theta_i=\sqrt{\rho_i}$ by the SLS-predicted tracking-tube radius, serving as a conservative constraint satisfaction margin due to dynamics error. Orientation constraints can be similarly robustified. \logicfix{Thm.~\ref{thm:error_bound} is integrated with the SLS tube for robust control in Sec.~\ref{subsec:exp_deformable}.}

% From $\Phi_P$ and $\mathcal W_t$, we can compute a per-time-step keypoint reachable tube of radius $r_{i,t}$ for each constrained keypoint $p_{i,t}$. For a learned distance constraint
% \begin{equation}
% g^{\mathrm{dist}}*{i,t}(P_t;\rho_i)
% =
% \rho_i - |p*{a_i,t}-c_{b_i,t}|*2^2
% \le 0,
% \qquad
% \theta_i := \sqrt{\rho_i},
% \end{equation}
% we replace the nominal constraint in (\ref{eq:planning_problem}) with the tightened constraint
% \begin{equation}
% g^{\mathrm{rob}}*{i,t}(P_t)
% :=
% \left(\theta_i + r_{i,t}\right)^2
% -
% |p_{a_i,t}-c_{b_i,t}|_2^2
% \le 0.
% \end{equation}
% Thus, the nominal planned keypoint must remain farther from the learned constraint boundary by an amount equal to the SLS-predicted tracking tube. When online keypoint feedback is available, the corresponding SLS feedback correction can be applied as $u_t=\bar u_t+\delta u_t$. In the hardware rope experiments, real-time keypoint feedback is not used; we execute the nominal controls $\bar u_t$ but retain the SLS tube as an offline robustness margin. This makes the planner conservative with respect to calibrated dynamics error without manually enlarging the learned constraint parameter $\theta_i$ itself.

% At execution time, the planned keypoint trajectory is tracked using a
% heading controller that advances waypoints as the robot approaches each
% planned state within a proximity threshold.
% \textcolor{red}{TODO: planning design.}
	
%===============================================================================
\section{Experiments}
\label{sec:experiments}

\begin{table*}[t]
    \centering
    \small
    \caption{
    Full numerical summary of keypoint constraint learning and planning results.
    ``Demo pts.'' are trajectory points used for inverse-KKT constraint learning.
    ``Val num.'' are the number of planning validation trials.
    \logicfix{``Clearance error'' is the absolute error in the recovered distance-clearance radius; for tasks also with orientation constraints, this reports the distance-clearance error.}
    ``Const. safe rate'' is satisfaction rate of the demonstration-inferred constraint during planning, 
    ``Obs. safe rate'' is satisfaction of the ground-truth obstacle-safety check, 
    and ``Goal reach'' is terminal task completion.
    Values reported as mean $\pm$ standard deviation are computed across evaluation batches.\vspace{-5pt}
    }
    \label{tab:numerical_results_summary}
    \resizebox{\textwidth}{!}{
    \begin{tabular}{lllccccccccc}
        \toprule
        Experiment & Setting & Method
        & Demo pts.
        & Val num.
        & Dyn. eval loss
        & KCL runtime (s)
        & Final stat.
        & \logicfix{Clearance err. (m)}
        & Const. Safe rate
        & obs. safe rate
        & Goal reach \\
        \midrule

        \multicolumn{11}{l}{\textbf{Dubins-car obstacle avoidance}} \\
        Dubins & Simulation & \MethodOurs  & 1640 & 1000 & 0.0028 & 2.3 & 2.09 & $0.011$ & $1.0\pm0.000$ & $1.0\pm0.000$ & $1.0\pm0.000$ \\
        Dubins & Simulation & \BaseOne     & 1640 & 1000 & \NA & \NA & \NA & \NA & $0.184\pm0.017$ & $0.997\pm0.005$ & $1.0\pm0.000$\\
        Dubins & Simulation & \BaseTwo     & 1640 & 1000 & 0.0752 & \NA & \NA & \NA & $0.0\pm0.000$ & $0.07\pm0.021$ & $1.0\pm0.000$\\
        Dubins & Simulation & \BaseThree   & 1640 & 1000 & \NA & \NA & \NA & \NA & $0.0\pm0.000$ & $0.608\pm0.063$ & $1.0\pm0.000$\\
        Dubins & Real       & \MethodOurs  & 815 & 5 & 0.0038259 & 4.1 & 0.2429 & $0.021$ & $1.0\pm 0.000$ & $1.0\pm0.000$ & $1.0\pm0.000$\\
        \midrule

        \multicolumn{11}{l}{\textbf{KUKA arm manipulation}} \\
        KUKA & Simulation & \MethodOurs  & 465 & 1000 & 0.002135 & 3.7 & 0.42491 & $0.03$& $1.0\pm0.000$ & $1.0\pm0.000$ & $1.0\pm0.000$ \\
        KUKA & Simulation & \BaseOne     & 465 & 1000 & \NA & \NA      & \NA      & \NA               & $0.635 \pm 0.051$   & $0.971 \pm 0.023$ & $1.0\pm 0.000$\\
        KUKA & Simulation & \BaseTwo     & 465 & 1000 & 0.018986 & \NA      & \NA      & \NA            & $0.003\pm 0.007$      & $0.021 \pm 0.014$ & $1.0\pm 0.000$\\
        KUKA & Simulation & \BaseThree   & 465 & 1000 & \NA & \NA      & \NA      & \NA              & $0.649\pm 0.051$     & $0.977\pm 0.016$ & $1.0\pm 0.000$\\
        KUKA & Real       & \MethodOurs  & 1004 & 5 & 0.11283 & 4.9 & 0.89369 & $0.045$      & $1.0\pm 0.000$         & $1.0\pm 0.000$ & $1.0\pm 0.000$\\
        \midrule

        \multicolumn{11}{l}{\textbf{Rope obstacle avoidance}} \\
        Rope obstacle & Simulation & \MethodOurs & 600 & 100 & 0.035998 & 2.5 & 0.51843 &$ 0.032$ & $0.95\pm 0.05$ & $0.92\pm 0.12$ & $1.0\pm 0.000$\\
        Rope obstacle & Simulation & \BaseTwo     & 600 & 100 & 0.001144 & \NA      & \NA      & \NA            & $0.38\pm 0.11$      & $0.82 \pm 0.13$ & $0.0\pm 0.000$\\
        Rope obstacle & Simulation & \BaseThree   & 600 & 100 & \NA & \NA      & \NA      & \NA              & $0.19\pm 0.14$     & $0.24\pm 0.15$ & $1.0\pm 0.000$\\
        Rope obstacle & Real       & \MethodOurs & 1240 & 6 & 0.00932 & 2.5 & 1.533 & $0.029$                & $1.0\pm 0.000$ & $1.0\pm 0.000$ & $1.0\pm 0.000$ \\
        \midrule

        \multicolumn{11}{l}{\textbf{Multi-stage }} \\
        \logicfix{Rope knot tying} & Simulation & \MethodOurs & 70 & \NA & 0.0047 & 2.3 & 0.0097869 & \NA & \NA & \NA & \NA \\
        Cup picking and pouring & Real & \MethodOurs & 375 & 5 & 0.002504 & 3.5 & 0.53715 & $0.0201$ & $1.0\pm 0.000$ & $1.0\pm 0.000$& $1.0\pm 0.000$\\
        Hammer placing & Real & \MethodOurs & 390 & 5 & 0.015228 & 5.4 & 0.58586 & $0.0042$ & $1.0\pm 0.000$ & $1.0\pm 0.000$& $1.0\pm 0.000$\\
        \bottomrule
    \end{tabular}
    }\vspace{-15pt}
\end{table*}

% 2.40^\circ,
% 1.29^\circ,

% We evaluate KCL on constrained robotic tasks of increasing visual and physical complexity: Dubins-car obstacle avoidance, KUKA arm manipulation, and deformable-object manipulation. These experiments test whether constraints can be inferred from keypoint--mask trajectories in simulated and real videos, and how the learned constraints can be used for planning. We report learned constraint visualizations comparing to ground truth when available. The constraints are then used in a downstream planner to design safe trajectories.
To show that keypoint--mask trajectories can support constraint inference and downstream safe planning, we evaluate KCL on Dubins-car obstacle avoidance, 7-DOF KUKA arm manipulation, deformable rope manipulation and two multi-stage human demonstration tasks in simulation and hardware. 
\logicfix{We report inverse-learning metrics (KKT stationarity residual, learned-vs-ground-truth constraint error when available, and one-step dynamics loss) and planning metrics (learned-constraint satisfaction, obstacle safety, and goal reaching).} We compare KCL against three baselines on the Dubins-car, KUKA and the rope tasks. \textbf{K-VIL}~\cite{gao2023k} runs a keypoint-based movement imitation pipeline. \textbf{ReKep}~\cite{huang2024rekep} uses a VLM to generate relational keypoint constraint templates for planning. \textbf{MBIRL}~\cite{mbirl} learns a keypoint-space objective with unconstrained inverse RL and plans with learned dynamics, but infers no hard inequality constraints, testing whether soft objectives alone reproduce safe behavior. All baselines use identical demonstration splits, keypoint trajectories, and evaluation task conditions. \logicfix{Unless otherwise stated, we hold $\theta_J$ fixed and known; exceptions are for the rope and the Dubins objective-misspecification ablation, where cost parameters are jointly learned with the constraints.} For rope manipulation, we do not report K-VIL because it targets rigid-body navigation/manipulation and does not support deformable-object constraint parameterizations. The numerical results for all experiments and baselines are shown in Table~\ref{tab:numerical_results_summary}.
\looseness-1More results, including a comparison with \cite{ral20}, learning non-convex
constraints with KCL, and more figures are in ~\cite[App. C]{kcl_supplement}.

\subsection{Dubins-Car Visual Obstacle Avoidance}
We evaluate KCL on an obstacle-avoidance task (Fig.~\ref{fig:dubins_results}) in PyBullet simulation and on a real TurtleBot, using vehicle keypoints $p_{i,t}$ and obstacle-mask centroid $c_{\mathrm{obs},t}$. The learned constraint enforces ground-plane clearance:
$g^{\mathrm{obs}}_{i}(P_t;\rho_i) = \rho_i - \|\Pi_{xy}(p_{i,t}) - \Pi_{xy}(c_{\mathrm{obs},t})\|_2^2 \leq 0$,
where $\rho_i$ is the learned squared distance margin. KCL achieves perfect constraint satisfaction, obstacle safety, and goal reaching in both settings (Table~\ref{tab:numerical_results_summary}, Fig.~\ref{fig:dubins_results}). \logicfix{All baselines also reach the goal, but they do not reliably satisfy the demonstration-inferred clearance: K-VIL, ReKep, and MBIRL achieve $18.4\%$, $0\%$, and $0\%$ learned-constraint satisfaction, with corresponding obstacle-safe rates of $99.7\%$, $60.8\%$, and $7.0\%$.} Constraint learning takes $2.3\,\mathrm{s}$ in simulation and $4.1\,\mathrm{s}$ on hardware. 
\logicfix{To test robustness to objective mis-specification, we replace the fixed displacement cost with an unknown cost mixture $J=\sum_{k=1}^3 \alpha_k J_k$, jointly optimizing the weights $\alpha_k$ and constraint parameters. With squared keypoint displacement $J_1$ (the default in \eqref{eq:keypoint_forward_problem}), goal error $J_2$, and path length $J_3$ as candidates, solving} 
\eqref{eq:keypoint_inverse_problem_relaxed} yields $\alpha_1 = 1$, $\alpha_2=0.063$, and $\alpha_3=0$, with constraint error 0.012. \logicfix{Thus, KCL can recover the constraints and identify squared keypoint displacement as the dominant objective without assuming known weights.}

\begin{figure}[t]
    \centering

        \includegraphics[width=\linewidth]{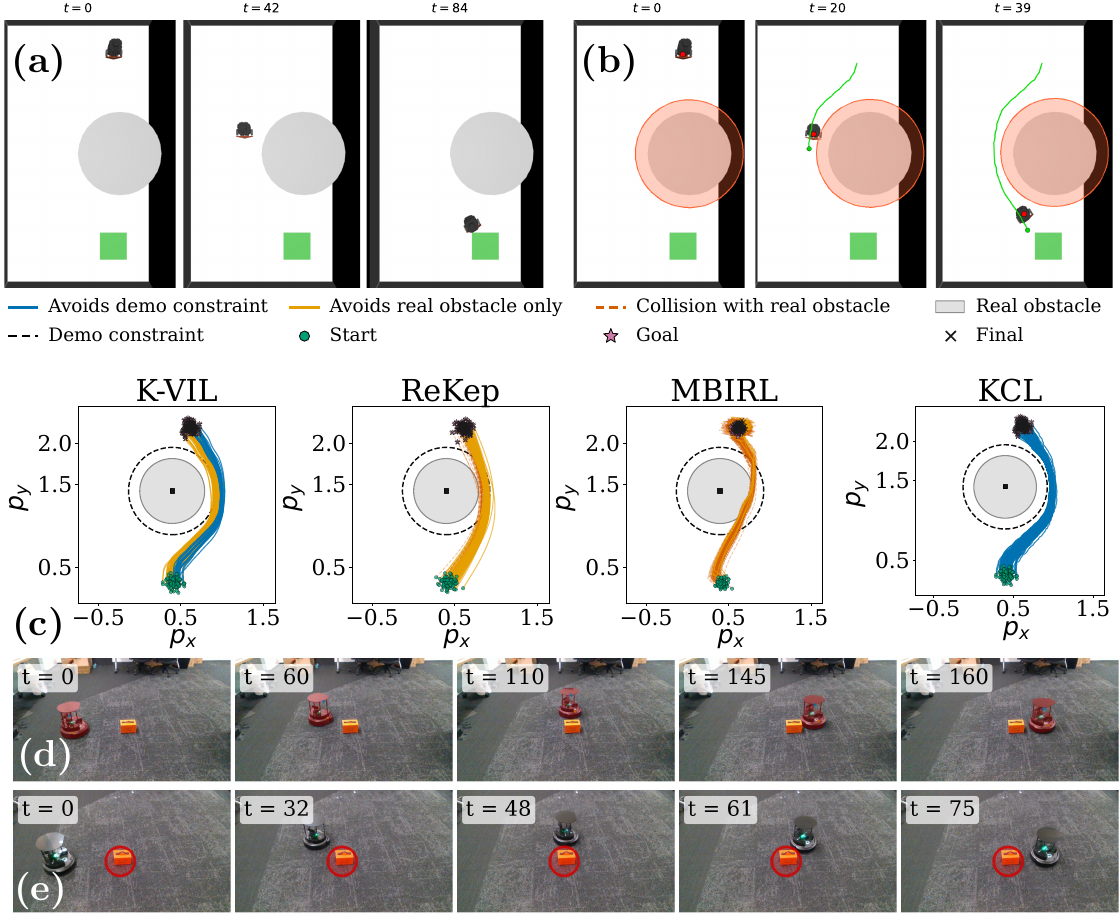}\vspace{-8pt}
    \caption{
    Dubins-car task: (a) Simulated demo and (b) plan (green trajectory is plan from KCL). \logicfix{(c) Rollouts of plans computed via baselines and KCL (only KCL achieves $100\%$ satisfaction of both the demonstration-inferred clearance and the ground-truth obstacle-safety check; see Tab.~\ref{tab:numerical_results_summary}).} (d) Real demo and (e) planning (red circle is constraint learned by KCL). \vspace{-11pt}
    }
    \label{fig:dubins_results}
\end{figure}

\subsection{KUKA Manipulation}
\looseness-1We next evaluate KCL on a KUKA manipulation task (Fig. \ref{fig:kuka_setup}), where the
demonstrator transports a cup to a target while avoiding an obstacle and
keeping the cup within a $30^\circ$ tilt bound, extending KCL to 3D motion
and orientation constraints. Obstacle avoidance uses distance
constraints~\eqref{eq:keypoint_mask_distance_constraint} with learned margin
$\rho_i$; cup orientation uses alignment
constraint~\eqref{eq:keypoint_orientation_constraint} on a cup keypoint pair
$(a,b)$ along world-frame vertical axis $\hat{e}_z$, with learned threshold
$\sigma_{a,b}\in[0,1]$. KCL achieves perfect constraint satisfaction, obstacle
safety, and goal reaching in simulation and hardware
(Table~\ref{tab:numerical_results_summary}, Fig.~\ref{fig:kuka_setup}), while K-VIL,
ReKep, and MBIRL achieve only $63.5\%$, $64.9\%$, and $0.3\%$ constraint
satisfaction and $97.1\%$, $97.7\%$, and $2.1\%$ obstacle safety, respectively.
Constraint learning takes $3.7\,\mathrm{s}$ in simulation and $4.9\,\mathrm{s}$ on hardware.
\begin{figure}[t]
    \centering
        \includegraphics[width=\linewidth]{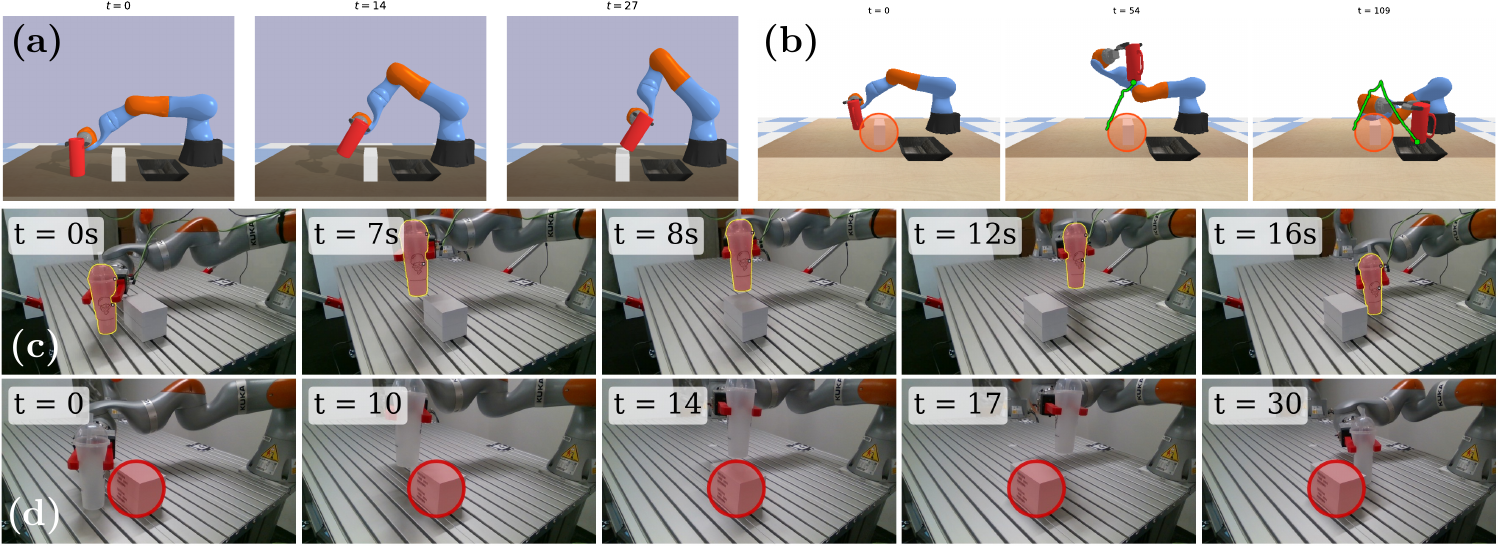}\vspace{-8pt}
    \caption{
    KUKA manipulation: Simulated (a) demo and (b) plan. Hardware (c) demo and (d) plan. In (b), (d), learned constraint (red circle) is overlaid with the true obstacle. In (b), the green trajectory is the planned motion of the bottom mug keypoint, which satisfies the learned obstacle constraint. The mug also satisfies the learned tilting threshold over the full trajectory.\vspace{-18pt}
    }
    \label{fig:kuka_setup}
\end{figure}

\subsection{Deformable Object Manipulation}
\label{subsec:exp_deformable}

% Finally, we evaluate the framework on deformable-object manipulation tasks. Deformable objects are a central motivation for our keypoint-based representation, since their full physical states are difficult to define, estimate, and model analytically. We consider two rope manipulation tasks: obstacle avoidance and knot tying. The rope obstacle-avoidance task includes both MuJoCo simulation and real-world demonstrations using two KUKA arms. The knot-tying task is evaluated in Drake/Filament simulation and uses a multi-stage constraint representation.
% Finally, we evaluate KCL on deformable objects, a key motivation for using keypoints, since full deformable states are hard to model and estimate. We consider rope obstacle avoidance and rope knot tying. To obtain demonstrations, rope obstacle avoidance uses MuJoCo simulation and real-world KUKA demonstrations, while rope knot tying is implemented in Drake with the Filament renderer and a multi-stage constraint representation~\cite{drake}.
We evaluate KCL on deformable-object tasks. We consider rope obstacle avoidance in MuJoCo and real-world KUKA demonstrations, and rope knot tying in Drake/Filament with a multi-stage constraint representation~\cite{drake}.

\noindent\textit{Rope Obstacle.\quad}
\looseness-1In the rope obstacle-avoidance task (Fig.~\ref{fig:rope_obstacle_setup}),
two arms move rope endpoints past a fixed obstacle; since straight-line
pulling would snag the rope, successful demonstrations form a wider V-shape,
requiring the learned constraints to act on intermediate keypoints rather
than endpoints alone. Per-keypoint clearance constraints
$g^{\mathrm{rope}}_{i}(P_t;\rho_i) = \rho_i - \|p_{i,t} -
c_{\mathrm{obs},t}\|_2^2 \leq 0$ let~\eqref{eq:keypoint_inverse_problem_relaxed}
identify which rope segments actively avoid the obstacle. \logicfix{We also learn the cost parameters $\theta_J$: with $\theta_{J,i}\in[0.9,1.1]$ and ground-truth uniform weights $\theta_{J,i}=1$, KCL recovers $\theta_{J,i}=0.9$ for all $i$. Because inverse-optimal cost weights are identifiable only up to a common positive scale, normalizing the recovered vector yields the ground-truth uniform weighting.}
Planning with KCL achieves obstacle-safe goal reaching for simulated and real-world trajectories (Table~\ref{tab:numerical_results_summary},
Fig.~\ref{fig:rope_obstacle_setup}), \logicfix{reproducing the outward-pulling behavior.} Learning takes $2.5\,\mathrm{s}$.

% \begingroup
% \color{red}
\logicfix{To show how Thm.~1 informs downstream robust planning, we apply its uncertainty-margin construction to the rope task.} We added
bounded observation perturbations of
$\epsilon_k=\epsilon_c=5\,\mathrm{mm}$ to the rope keypoints and obstacle-mask
centroid and re-ran KCL on the perturbed demonstration. The recovered squared
clearance changed from
\logicfix{$\rho_{\mathrm{clean}}=0.14559\,\mathrm{m}^2$ for the unperturbed data to $\hat{\rho}=0.14461\,\mathrm{m}^2$ after perturbation.}
For this single retained distance constraint, the perception component of
Thm.~1 gives
$B_{\mathrm{perc}}=7.71\times10^{-3}\,\mathrm{m}^2$. 
We also include the learned-dynamics term in Thm.~1,
$B_{\mathrm{dyn}}=\kappa_\rho\epsilon_f\bar{\nu}$.
\logicfix{Because the exact rope dynamics Jacobians and local inverse-KKT sensitivity constants are unavailable, we use illustrative design values $\epsilon_f=0.02$, $\kappa_\rho=0.10$, and $\bar{\nu}=1.0$ to demonstrate the construction, where $\epsilon_f$ represents a local learned-dynamics Jacobian error budget.}
\logicfix{With these design values,} this gives
$B_{\mathrm{dyn}}=2.00\times10^{-3}\,\mathrm{m}^2$ and hence the total
\logicfix{illustrative Thm.~1 margin}
$B_\rho
=
B_{\mathrm{perc}}+B_{\mathrm{dyn}}
=
9.71\times10^{-3}\,\mathrm{m}^2$. 
We therefore inflate the learned constraint according to
$\rho_{\mathrm{T1}}=\hat{\rho}+B_\rho$, increasing the required clearance from
$\hat{\theta}=0.3803\,\mathrm{m}$ to
$\theta_{\mathrm{T1}}=0.3928\,\mathrm{m}$.

Theorem~1 inflates the recovered constraint to account for keypoint/mask-centroid perception error and learned-dynamics Jacobian error. \highvaluefix{Let $i^\star$ denote the retained rope clearance constraint and define
$r_t^{\mathrm{SLS}}:=r_{i^\star,t}$ using \eqref{eq:tube}.}
Separately, this SLS tube, computed from dynamics prediction error, captures accumulated rollout uncertainty during planning. The robust constraint is
\highvaluefix{$\left\|\bar p_{a_{i^\star},t}-\bar c_{\mathrm{obs}}\right\|_2
\ge \theta_{\mathrm{T1}}+r_t^{\mathrm{SLS}}$.}
\logicfix{Thus, this experiment illustrates how the planner can combine margins for perception error, constraint-recovery dynamics error, and downstream rollout uncertainty.}
% The Theorem~1 inflation accounts for uncertainty in the recovered constraint
% caused by keypoint/mask-centroid perception error and inaccuracies in the
% learned dynamics Jacobians. We then separately use the SLS tube
% $r_t^{\mathrm{SLS}}$, computed from learned-dynamics prediction error, to
% account for accumulated rollout uncertainty during planning. The resulting
% robust planning constraint is
% $\left\|
% \bar p_{K2,t}-\bar c_{\mathrm{obs}}
% \right\|_2
% \ge
% \theta_{\mathrm{T1}}+r_t^{\mathrm{SLS}}$. 
% Thus, the planner simultaneously accounts for perception-induced constraint
% uncertainty, learned-dynamics error during constraint recovery, and accumulated
% dynamics error during downstream planning.

% \endgroup

\begin{figure}[t]
    \centering
        \includegraphics[width=\linewidth]{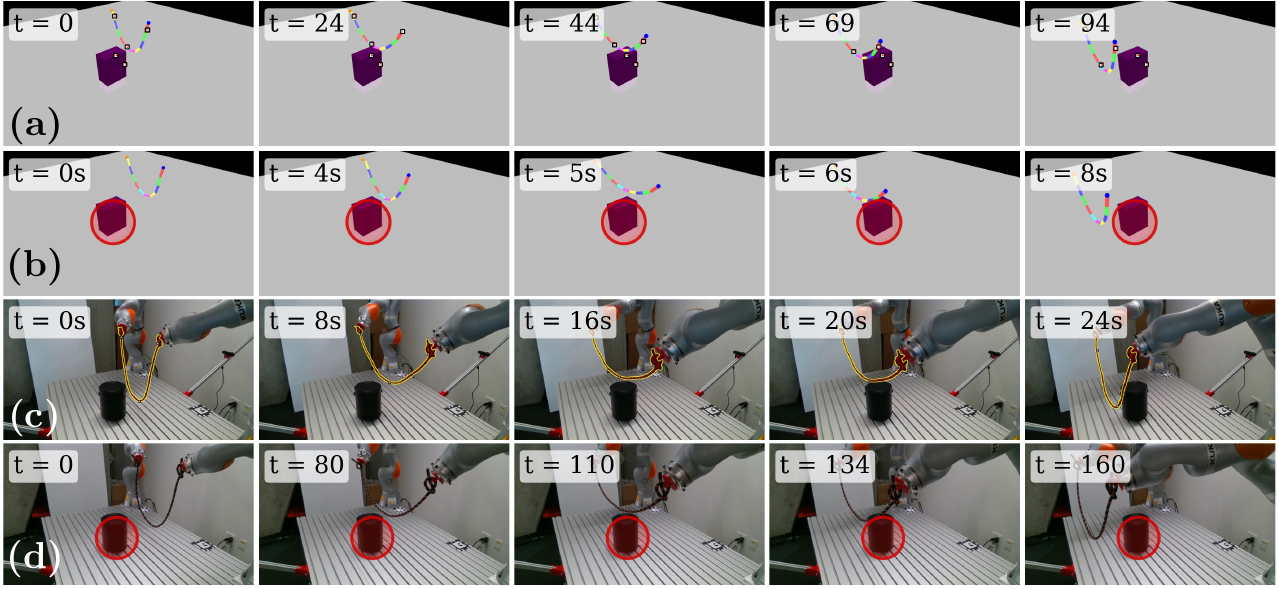}\vspace{-8pt}
    \caption{
    Rope obstacle avoidance: Simulated (a) demo and (b) plan computed by KCL. Hardware (c) demo and (d) plan computed by KCL. In (b), (d), learned constraint (light red circle) is overlaid with the ground-truth obstacle.\vspace{-15pt}
    }
    \label{fig:rope_obstacle_setup}
\end{figure}

\looseness-1\noindent\textit{Rope Knot Tying.\quad}
We evaluate KCL on rope knot tying in Drake/Filament simulation~\cite{drake} (Fig. \ref{fig:rope_knot_setup}).
The task contains distinct phases---loop formation, endpoint threading, and tightening---with changing constraint structure, making a single constraint insufficient. 
KCL instead jointly recovers latent geometric constraints,
latent semantic phase boundaries, and KKT-consistent explanations of expert
behavior, without assuming phase boundaries are known \textit{a priori}. 
\highvaluefix{The rope is represented by five keypoints: left endpoint $p_{1,t}$, right
endpoint $p_{5,t}$, and three interior loop-forming keypoints
$p_{2,t},p_{3,t},p_{4,t}$, with loop centroid
$c_{\mathrm{loop},t}=\tfrac{1}{3}(p_{2,t}+p_{3,t}+p_{4,t})$ and
demonstrated final position $p_{\mathrm{final}}:=p_{5,T}$.}
We partition the demonstration into three stages with stage-specific
constraints.
Stage~1 (\emph{insertion}, $[0,t_1-1]$) has no learned inequality.
\highvaluefix{Stage~2 (\emph{loop proximity}, $[t_1,t_2]$) constrains the right endpoint
to remain within squared distance $\rho_{\mathrm{loop}}$ of the loop centroid:
$g^{\mathrm{loop}}(P_t;\rho_{\mathrm{loop}}):=\|p_{5,t}-c_{\mathrm{loop},t}\|_2^2
-\rho_{\mathrm{loop}}\leq 0$.}
\notationfix{Stage~3 (\emph{final pull}, $[t_2+1,T]$) constrains the right endpoint}
\highvaluefix{to remain near $p_{\mathrm{final}}$ via threshold $\rho_{\mathrm{far}}$:
$g^{\mathrm{final}}(P_t;\rho_{\mathrm{far}}):=\|p_{5,t}-p_{\mathrm{final}}\|_2^2
-\rho_{\mathrm{far}}\leq 0$.
The learned thresholds admit direct geometric interpretations as radii:
$\theta_{\mathrm{loop}}=\sqrt{\rho_{\mathrm{loop}}}$ and
$\theta_{\mathrm{far}}=\sqrt{\rho_{\mathrm{far}}}$.}

\looseness-1Crucially, stage boundaries $t_1,t_2$ are \emph{co-solved} with constraint parameters \highvaluefix{$\rho_{\mathrm{loop}},\rho_{\mathrm{far}}$} in the MILP using binary phase indicators and Big-M encoding, rather than specified a priori. Figure~\ref{fig:rope_knot_setup} shows that solving~\eqref{eq:keypoint_inverse_problem_relaxed} recovers these quantities near the ground truth, jointly identifying \emph{what} constraints define knot-tying and \emph{when} they are active while satisfying local KKT optimality. Constraint learning takes $2.3\,\mathrm{s}$. KCL can also select among candidate stage sequences rather than assuming a fixed sequence or number of stages. We consider three hypotheses: \emph{loop-then-final} (ground truth), \emph{z-cross-then-final}, and \emph{final-only}. Using one-hot binary variables to encode these alternatives, KCL jointly selects the sequence, stage boundaries, and constraints. It recovers the correct \emph{loop-then-final} sequence, with stationarity residual $5.206\times10^{-5}$ versus $5.24\times10^{-5}$ and $5.27\times10^{-5}$ for the distractors.

\begin{figure*}[t]
        \centering
        \includegraphics[width=\linewidth]{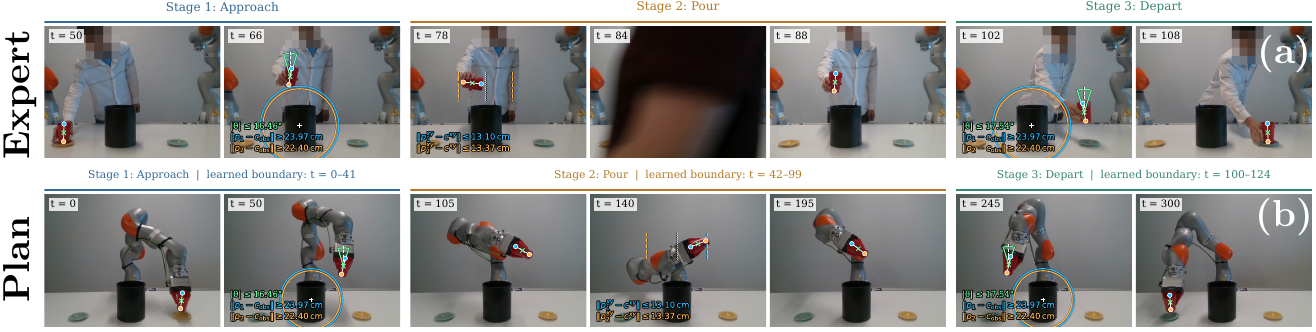}\vspace{-12pt}
        \caption{\looseness-1\textbf{(a)} Keypoint constraints learned from a video of a human lifting, pouring, and placing a mug, where low-level constraints between objects are learned together with the task segmentation (Stages 1-3). We also have occlusions in the demonstration. \textbf{(b)} KCL plans a safe trajectory using the constraints.}
    \label{fig:multi_stage}\vspace{-15pt}
\end{figure*}

\begin{figure}[t]
    \centering
        \includegraphics[width=\linewidth]{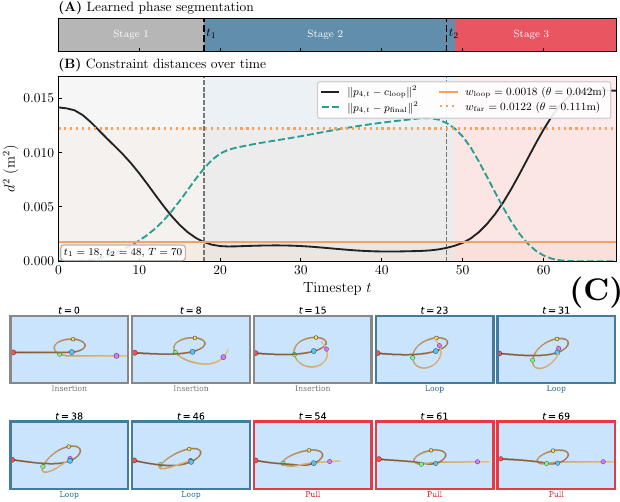}\vspace{-10pt}
        \caption{Knot tying: KCL jointly recovers (a) stage boundaries $t_1{=}18$, $t_2{=}48$
(black dashed) and (b) constraint thresholds $\theta_{\rm loop}{=}0.042$\,m
(orange solid), $\theta_{\rm far}{=}0.111$\,m (orange dotted) by solving
a single MILP, matching the ground-truth decomposition without manual
phase annotation. (c) Multi-stage decomposition of demonstration via learned constraints.\vspace{-22pt}}
    \label{fig:rope_knot_setup}
\end{figure}

\subsection{Multi-stage Human Manipulation}
% Additionally, we conduct two multi-stage manipulation experiments using human demonstrations to verify that KCL can learn constraints without requiring robot-generated demonstrations. 

% The first pouring experiment is divided into three manipulation stages: 
% 1) pick up the cup and move it above the container; 
% 2) pour the contents of the cup into the container; and 
% 3) place the cup on the other side of the container. 
% The cup and container are each represented by two keypoints and one centroid, denoted by 
% $p_{1,t}, p_{2,t}, c_{\mathrm{cup},t}$ and 
% $p_{3,t}, p_{4,t}, c_{\mathrm{container},t}$, respectively. 
% The stage-specific constraints are defined as follows. 
% In Stage 1, the cup should remain approximately upright:
% $g^{\mathrm{angle}}_1(P_t;\theta_{\mathrm{angle},1})
% :=
% \mathrm{ang}(p_{1,t}-p_{2,t},p_{1,0}-p_{2,0})
% -\theta_{\mathrm{angle},1}
% \leq 0.$
% In Stage 2, the cup should remain sufficiently close to the container in the $xy$-plane (the plane parallel to the desk):
% $g^{xy}_i(P_t;w_{\mathrm{xy},i})
% :=
% \|p_{i,t}^{xy}-c_{\mathrm{container},t}^{xy}\|_2^2
% -w_{\mathrm{xy},i}
% \leq 0.$
% In Stage 3, a similar upright-orientation constraint is imposed:
% $g^{\mathrm{angle}}_2(P_t;\theta_{\mathrm{angle},2})
% :=
% \mathrm{ang}(p_{1,t}-p_{2,t},p_{1,0}-p_{2,0})
% -\theta_{\mathrm{angle},2}
% \leq 0.$
Additionally, we conduct two multi-stage manipulation experiments using human demonstrations, showing that KCL can learn constraints without robot-generated data (Fig. \ref{fig:hammer}, \ref{fig:multi_stage}). 
Since control inputs are not directly available from the human videos, we model the demonstrated motion with a 3D single-integrator and reconstruct the equivalent input sequence.
The first pouring experiment has three stages: 1) move the cup above the container; 2) pour into the container; and 3) place the cup on the opposite side. The cup and container are represented by keypoints $p_{1,t},p_{2,t}$ and $p_{3,t},p_{4,t}$, with centroids $c_{\mathrm{cup},t}$ and $c_{\mathrm{container},t}$. \highvaluefix{Define
$\mathrm{ang}(a,b):=\arccos((a^\top b)/(\|a\|_2\|b\|_2))$
as the angle between nonzero vectors $a$ and $b$.}
In Stages 1 and 3, the cup is constrained to remain approximately upright:
$g^{\mathrm{angle}}_j(P_t;\theta_{\mathrm{angle},j})
:=
\mathrm{ang}(p_{1,t}-p_{2,t},p_{1,0}-p_{2,0})
-\theta_{\mathrm{angle},j}
\leq 0,\quad j\in\{1,2\}$. 
In Stage 2, the cup must remain close to the container in the desk-parallel $xy$-plane:
\highvaluefix{$g^{xy}_i(P_t;\rho_{\mathrm{xy},i})
:=
\|p_{i,t}^{xy}-c_{\mathrm{container},t}^{xy}\|_2^2
-\rho_{\mathrm{xy},i}
\leq 0$.}
An obstacle-avoidance constraint,
\highvaluefix{$g^{\mathrm{obs}}_i(P_t;\rho_{\mathrm{obs},i})
:=
-\|p_{i,t}-c_{\mathrm{container},t}\|_2^2
+\rho_{\mathrm{obs},i}
\leq 0$}
is also enforced for all three stages. 
The hammer experiment (Fig. \ref{fig:hammer}) has three stages: 1) move the hammer around an obstacle; 2) rotate its handle toward the shelf; and 3) place it on the shelf. The hammer, obstacle, and shelf are each represented by two keypoints and one centroid. In Stage 1, we impose an obstacle-avoidance constraint:
% The second hammer experiment is also divided into three manipulation stages:
% 1) pick up the hammer and move it around the obstacle; 
% 2) rotate the hammer until its handle is oriented toward the shelf; and 
% 3) place the hammer on the shelf. 
% The hammer, obstacle and shelf are each represented by two keypoints and one centroid. 
% The stage-specific constraints are defined as follows. 
% In Stage 1, an obstacle-avoidance constraint is imposed:
\highvaluefix{$g^{\mathrm{obs}}_i(P_t;\rho_{\mathrm{obs},i})
:=
-\|p_{i,t}-c_{\mathrm{obs},t}\|_2^2
+\rho_{\mathrm{obs},i}
\leq 0$.}
Stage 2 does not impose an additional learned constraint. 
\highvaluefix{Let $t_3:=t_2+1$ denote the first time index of Stage~3.}
In Stage 3, the hammer should maintain the desired orientation while remaining sufficiently close to the shelf:
\highvaluefix{$g^{\mathrm{angle}}_1(P_t;\theta_{\mathrm{angle},1})
:=
\mathrm{ang}(p_{1,t}-p_{2,t},p_{1,t_3}-p_{2,t_3})
-\theta_{\mathrm{angle},1}
\leq 0$,}
and
\highvaluefix{$g^{xy}_i(P_t;\rho_{\mathrm{xy},i})
:=
\|p_{i,t}^{xy}-c_{\mathrm{shelf},t}^{xy}\|_2^2
-\rho_{\mathrm{xy},i}
\leq 0$.} 
As for knot-tying, stage boundaries are co-solved with the constraint parameters in the MILP. The learned constraints are then used to plan and execute the tasks with a KUKA arm. Across five trials per task with randomized initial poses, all plans satisfy the learned constraints and reach each stage goal. Results are reported in Table~\ref{tab:numerical_results_summary}, with demonstrations and plan timelapses in \logicfix{Figs.~\ref{fig:hammer} and~\ref{fig:multi_stage}}. 
% As in the rope knot-tying experiment, stage boundaries are co-solved with the constraint parameters in the MILP. We then use the learned constraints to plan and execute the tasks with a KUKA arm. Across five trials per task with randomized initial poses, all plans satisfy the learned constraints and reach each stage goal. Quantitative results are reported in Table~\ref{tab:numerical_results_summary}, with demonstrations and planning timelapses in Fig.~\ref{fig:multi_stage}.
% As in the rope knot-tying experiment, the stage boundaries in these experiments are co-solved with the constraint parameters within the MILP. 
% After learning the constraints, task planning is performed in real world using a KUKA arm to execute the tasks. 
% We evaluate five planning trials with randomly sampled initial poses for each task, and all trials successfully satisfy the learned constraints and reach the goal in each stage. 
% The quantitative results are reported in Table~\ref{tab:numerical_results_summary}, while demonstration and planning timelapses are shown in Fig.~\ref{fig:multi_stage}.

\vspace{-5pt}
\subsection{Robustness Experiments}

% \looseness-1\noindent\textit{Keypoint Dropout Test:\quad}
% A challenge for keypoint-based methods is imperfect tracking: individual keypoints may be lost, and camera occlusion may temporarily hide relevant objects. We therefore evaluate KCL under keypoint dropout. First, we randomly remove $50\%$ of observations from the real-world Dubins-car demonstration (Fig. \ref{fig:dubins_results}d). Second, in the pouring demonstration (Fig.~\ref{fig:multi_stage}), we occlude the camera over a continuous interval, creating a contiguous gap. Missing keypoints are reconstructed using PCHIP interpolation~\cite{birkhoff1968piecewise}. We compare constraints learned from these reconstructed trajectories against those from the original trajectories. \logicfix{The componentwise learned-clearance errors are $(0.0000,0.0093)\,\mathrm{m}$ for random dropout and $(0.0000,0.0004)\,\mathrm{m}$ for continuous dropout, demonstrating KCL's robustness to missing observations.}

\looseness-1\noindent\textit{Keypoint Dropout Test:\quad}
A challenge for keypoint-based methods is imperfect tracking: individual keypoints may be lost, and camera occlusion may temporarily hide relevant objects. We test KCL by randomly dropping $50\%$ of observations in the real-world Dubins-car demonstration (Fig.~\ref{fig:dubins_results}d) and occluding a continuous interval in the pouring demonstration (Fig.~\ref{fig:multi_stage}). After reconstructing missing keypoints via PCHIP interpolation~\cite{birkhoff1968piecewise}, \logicfix{the componentwise learned-clearance errors relative to complete trajectories are $(0.0000,0.0093)\,\mathrm{m}$ and $(0.0000,0.0004)\,\mathrm{m}$, respectively, demonstrating robustness to random and continuous dropout.}

\begin{wraptable}{r}{0.38\columnwidth}
    \centering\vspace{-23pt}
    \caption{Constraint recovery noise robustness.\vspace{-8pt}%Robustness of constraint recovery to Gaussian keypoint noise
    %(mean $\pm$ stdev over 10 trials; $\sigma$ denotes per-coordinate noise stdev).
    }
    \label{tab:noise}
    \resizebox{\linewidth}{!}{%
    \begin{tabular}{ccc}
        \toprule
        \highvaluefix{\textbf{$\sigma_{\rm noise}$}} & \textbf{\# Trials} & \highvaluefix{\textbf{Clearance RMSE (m)}} \\
        \midrule
        0.00 & 10 & $0.0000 \pm 0.0000$ \\
        0.01 & 10 & $0.0146 \pm 0.0022$ \\
        0.05 & 10 & $0.0738 \pm 0.0078$ \\
        0.10 & 10 & $0.1220 \pm 0.0205$ \\
        0.15 & 10 & $0.1704 \pm 0.0423$ \\
        0.20 & 10 & $0.1912 \pm 0.0250$ \\
        \bottomrule
    \end{tabular}%
    }\vspace{-8pt}
\end{wraptable}

\noindent\textit{Keypoint Noise Test:\quad}
To evaluate robustness to other perception errors, we perform constraint recovery on the Dubins car under increasing Gaussian keypoint noise over 10 random seeds. \highvaluefix{Here, $\sigma_{\rm noise}$ denotes the per-coordinate noise standard deviation, and the reported RMSE is the error in the recovered clearance distance.} As shown in Tab.~\ref{tab:noise}, recovery error shows graceful degradation rather than abrupt sensitivity to small perception errors. We also observed no active constraint flipping due to the noise.
% To test the robustness of KCL againt other sources of detection noise, constraint recovery experiments are performed
% under increasing Gaussian keypoint noise for Dubins car over 10 random seeds. As shown in Fig~\ref{fig:noise_curve}, the recovered constraint parameters degrade smoothly as the noise magnitude increases, rather than exhibiting abrupt failure under small perturbations. The constraint-recovery RMSE grows approximately proportionally with the injected noise over the tested range, indicating graceful degradation rather than strong sensitivity to small perception errors.
\begin{wraptable}{r}{0.27\columnwidth}
\vspace{-15pt}
\centering
\caption{Solve times (10 trials).\vspace{-6pt}}
\label{tab:solve_time_scaling}
\setlength{\tabcolsep}{2pt}
\renewcommand{\arraystretch}{0.4}
\scriptsize
\resizebox{\linewidth}{!}{%
\begin{tabular}{lccc}
\toprule
Task & $N$ & $M$ & Solve time(s) \\
\midrule
Nav.      & 10 & 5  & $0.56\!\pm\!.01$ \\
Nav.      & 20 & 10 & $0.75\!\pm\!.00$ \\
Nav.      & 50 & 50 & $2.05\!\pm\!.05$ \\
% \midrule
Rigid     & 8  & 4  & $0.07\!\pm\!.00$ \\
Rigid     & 16 & 8  & $0.11\!\pm\!.00$ \\
Rigid     & 50 & 50 & $0.37\!\pm\!.04$ \\
% \midrule
Rope & 5  & 2  & $0.03\!\pm\!.01$ \\
Rope & 10 & 4  & $0.04\!\pm\!.00$ \\
Rope & 50 & 50 & $1.23\!\pm\!.01$ \\
% \midrule
Knot      & 5  & 1  & $0.10\!\pm\!.00$ \\
Knot      & 10 & 1  & $0.21\!\pm\!.00$ \\
Knot      & 15 & 1  & $0.25\!\pm\!.00$ \\
\bottomrule
\end{tabular}%
}\vspace{-10pt}
\end{wraptable}

\noindent\textit{Solve Time Scaling:\quad} KCL uses mixed-integer programming, whose worst-case computational cost can grow exponentially with the number of integer variables. To assess its practical scalability, we measure solve times as the numbers of keypoints ($N$) and centroids ($M$) increase. As shown in Table~\ref{tab:solve_time_scaling}, mean solve times remain at or below $2.05$~s across all tested configurations, including those with $N=M=50$. These results indicate that KCL is scalable over the evaluated range.

% \endgroup
%===============================================================================
\vspace{-5pt}
\section{Conclusion}
\label{sec:conclusion}

We present KCL, a framework for learning task-relevant constraints from visual demonstrations for safe planning. KCL extracts keypoints and mask centroids from RGB-D demonstrations, learns keypoint-space dynamics, and infers geometric constraints via approximate KKT optimality in the reduced visual state space. The learned constraints are interpretable, directly enforceable during planning, and robustified using margins from dynamics error. Across navigation, rigid manipulation, and deformable manipulation tasks, KCL recovers meaningful constraints and produces safe goal-reaching plans in simulation and hardware. These results show that sparse visual keypoints can make constrained inverse optimal control practical for visuomotor control.

\vspace{-3pt}
\bibliographystyle{IEEEtran}
\bibliography{reference}

\clearpage
\appendix
\appendices

\section{Proofs}\label{app:proofs}

% \clearpage
\logicfix{For convenience, we restate Theorem~\ref{thm:error_bound} before the proof.}

\noindent\textbf{Theorem~\ref{thm:error_bound} (restated).}\ \begingroup\itshape
Consider the inverse problem \eqref{eq:keypoint_inverse_problem_relaxed}
\notationfix{on ideal reduced trajectories $\mathcal Z=\{z^q\}_{q=1}^D$ and perceived
approximations $\hat{\mathcal Z}=\{\hat z^q\}_{q=1}^D$, following
\eqref{eq:keypoint_demo}. Assume the constraint parameters are shared across time within the stage under consideration and
$\theta_J$ is known.} Suppose the perceived keypoints and centroids satisfy
$\|\hat p_{i,t}^q-p_{i,t}^q\|_2\le \epsilon_k$ and
$\|\hat c_{j,t}^q-c_{j,t}^q\|_2\le \epsilon_c$ for all $q,t,i,j$, and define
$r_d:=\epsilon_k+\epsilon_c$. Let $\rho\in\mathbb R^{n_d}$ and
$\sigma\in\mathbb R^{n_o}$ denote the learned distance-margin and
orientation-threshold parameters in
\eqref{eq:keypoint_mask_distance_constraint} and
\eqref{eq:keypoint_orientation_constraint}, respectively. Assume all
orientation-defining keypoint pairs satisfy
$\|p_{a,t}^q-p_{b,t}^q\|_2\ge \ell_{\min}>2\epsilon_k$, and define
$r_o:=2\epsilon_k/(\ell_{\min}-2\epsilon_k)$.

Let $(\rho^\star,\sigma^\star)$, $(\tilde\rho,\tilde\sigma)$, and
$(\hat\rho,\hat\sigma)$ denote the recovered parameters using, respectively:
exact perception and exact dynamics Jacobians; perturbed perception and exact
Jacobians at the perceived trajectory; and perturbed perception with learned
Jacobians. Assume these three problems share the same reduced active set and
that each active geometric parameter remains pinned by the same feature, i.e.,
$\rho_i^\star=\|p_{a_i,t_i}^{q_i}-c_{b_i,t_i}^{q_i}\|_2^2$,
$\tilde\rho_i=\|\hat p_{a_i,t_i}^{q_i}-\hat c_{b_i,t_i}^{q_i}\|_2^2$,
$\sigma_j^\star=\phi_j(P_{t_j}^{q_j})$, and
$\tilde\sigma_j=\phi_j(\hat P_{t_j}^{q_j})$.

\notationfix{Let $f^\star$ denote the exact keypoint-space dynamics map, and define
$\tilde A_t^q:=\frac{\partial f^\star}{\partial P}(\hat P_t^q,u_t^q)$,
$\tilde B_t^q:=\frac{\partial f^\star}{\partial u}(\hat P_t^q,u_t^q)$,
$\hat A_t^q:=\frac{\partial f_\psi}{\partial P}(\hat P_t^q,u_t^q)$, and
$\hat B_t^q:=\frac{\partial f_\psi}{\partial u}(\hat P_t^q,u_t^q)$.
Suppose the learned Jacobians satisfy
$\left\|\begin{bmatrix}
\hat A_t^q-\tilde A_t^q & \hat B_t^q-\tilde B_t^q
\end{bmatrix}\right\|_2\le \epsilon_f$ for all $q,t$.}
\highvaluefix{Let $F(y;\zeta)=0$, with
$y=(\rho,\sigma,\lambda,\nu,\ldots)$, denote the reduced active-set inverse-KKT
system. For the data $\tilde\zeta$ with perturbed perception and exact Jacobians,
define $\mathcal U_{\rm jac}$ as all data with the same perceived trajectory as
$\tilde\zeta$ and Jacobian perturbations satisfying
$\max_{q,t}\left\|\begin{bmatrix}
A_t^q-\tilde A_t^q & B_t^q-\tilde B_t^q
\end{bmatrix}\right\|_2\le\epsilon_f$. Assume that, for every
$\zeta\in\mathcal U_{\rm jac}$, the reduced KKT system has a locally unique solution
with the same active set, that
$\bar\nu:=\sup_{\zeta\in\mathcal U_{\rm jac}}\|\nu(\zeta)\|_2<\infty$, and that there
exist $\kappa_\rho,\kappa_\sigma>0$ such that
$\|\rho(\zeta)-\rho(\tilde\zeta)\|_2
\le
\kappa_\rho
\|F(y(\tilde\zeta);\zeta)-F(y(\tilde\zeta);\tilde\zeta)\|_2$
and
$\|\sigma(\zeta)-\sigma(\tilde\zeta)\|_2
\le
\kappa_\sigma
\|F(y(\tilde\zeta);\zeta)-F(y(\tilde\zeta);\tilde\zeta)\|_2$.}
Finally, let
$\tilde d_{\max}:=\max_i
\|\hat p_{a_i,t_i}^{q_i}-\hat c_{b_i,t_i}^{q_i}\|_2$. Then
\begin{equation}
    \|\hat\rho-\rho^\star\|_2
    \le
    \sqrt{n_d}\left(2\tilde d_{\max}r_d+r_d^2\right)
    +
    \kappa_\rho\epsilon_f\bar\nu,
    \vspace{-11pt}
\end{equation}

\begin{equation}
    \|\hat\sigma-\sigma^\star\|_2
    \le
    \sqrt{n_o}\,r_o
    +
    \kappa_\sigma\epsilon_f\bar\nu .
\end{equation}

\endgroup

\begin{proof}
We first bound the perception-induced error in the active distance
features. For any active keypoint--centroid pair, let
$d=\|p-c\|_2$ and $\hat d=\|\hat p-\hat c\|_2$. By the reverse triangle
inequality,
\[
    |\hat d-d|
    \le
    \|(\hat p-\hat c)-(p-c)\|_2
    \le
    \|\hat p-p\|_2+\|\hat c-c\|_2
    \le r_d .
\]
Therefore
\[
    |\hat d^2-d^2|
    =
    |\hat d-d|(\hat d+d)
    \le
    r_d(2\hat d+r_d).
\]
Since each distance parameter is pinned by the same active feature before
and after perception perturbation,
\[
    |\tilde\rho_i-\rho_i^\star|
    \le
    \logicfix{2\tilde d_{\max}r_d+r_d^2}.
\]
Thus
\[
    \|\tilde\rho-\rho^\star\|_2
    \le
    \sqrt{n_d}(2\tilde d_{\max}r_d+r_d^2).
\]

Next, consider an active orientation feature. Let
$v=p_a-p_b$ and $\hat v=\hat p_a-\hat p_b$. Then
$\|\hat v-v\|_2\le 2\epsilon_k$. Along the segment
$v_s=v+s(\hat v-v)$, $s\in[0,1]$, we have
$\|v_s\|_2\ge \ell_{\min}-2\epsilon_k$. For
\[
    \phi(v)=\left(\frac{v^\top\hat e}{\|v\|_2}\right)^2,
\]
a direct calculation gives
\[
    \|\nabla_v\phi(v)\|_2
    =
    \frac{2}{\|v\|_2}
    |\alpha(v)|\sqrt{1-\alpha(v)^2}
    \le
    \frac{1}{\|v\|_2},
    \qquad
    \alpha(v):=\frac{v^\top\hat e}{\|v\|_2}.
\]
Hence, by the mean value theorem,
\[
    |\phi(\hat v)-\phi(v)|
    \le
    \frac{\|\hat v-v\|_2}{\ell_{\min}-2\epsilon_k}
    \le
    \frac{2\epsilon_k}{\ell_{\min}-2\epsilon_k}
    =
    r_o.
\]
Since each orientation threshold is pinned by the same active feature,
\[
    \|\tilde\sigma-\sigma^\star\|_2
    \le
    \sqrt{n_o}r_o.
\]

It remains to bound the additional effect of using learned dynamics
Jacobians. \highvaluefix{For the learned-Jacobian data $\hat\zeta\in\mathcal U_{\rm jac}$,
define
\[
    \Delta F_{\rm dyn}
    :=
    F(y(\tilde\zeta);\hat\zeta)-F(y(\tilde\zeta);\tilde\zeta).
\]}
By the assumed local KKT sensitivity,
\[
    \|\hat\rho-\tilde\rho\|_2
    \le
    \kappa_\rho\|\Delta F_{\rm dyn}\|_2,
    \qquad
    \|\hat\sigma-\tilde\sigma\|_2
    \le
    \kappa_\sigma\|\Delta F_{\rm dyn}\|_2.
\]
The dynamics contribution to stationarity contains blocks
$-(\tilde A_t^q)^\top\nu_t^q$ and
$-(\tilde B_t^q)^\top\nu_t^q$. Replacing
$(\tilde A_t^q,\tilde B_t^q)$ with $(\hat A_t^q,\hat B_t^q)$ therefore
perturbs the stacked stationarity residual by at most
\[
    \|\Delta F_{\rm dyn}\|_2
    \le
    \highvaluefix{\epsilon_f\|\nu(\tilde\zeta)\|_2}
    \le
    \epsilon_f\bar\nu,
\]
\highvaluefix{where the last inequality follows from $\tilde\zeta\in\mathcal U_{\rm jac}$.}
Thus
\[
    \|\hat\rho-\tilde\rho\|_2\le \kappa_\rho\epsilon_f\bar\nu,
    \qquad
    \|\hat\sigma-\tilde\sigma\|_2\le \kappa_\sigma\epsilon_f\bar\nu.
\]
The claimed bounds follow by the triangle inequality:
\[
    \|\hat\rho-\rho^\star\|_2
    \le
    \|\tilde\rho-\rho^\star\|_2+\|\hat\rho-\tilde\rho\|_2,
\]
and
\[
    \|\hat\sigma-\sigma^\star\|_2
    \le
    \|\tilde\sigma-\sigma^\star\|_2+\|\hat\sigma-\tilde\sigma\|_2.
\]
\end{proof}

\section{Constraint Pruning}\label{app:pruning}

\notationfix{The inverse problem~\eqref{eq:keypoint_inverse_problem_relaxed} is formulated
over distance candidates $(a_i,b_i)$, $i=1,\ldots,n_d$, and orientation
candidates $(a_j,b_j,\hat e_j)$, $j=1,\ldots,n_o$, which can be numerous when
$N$ and $M$ are large. After solving, most candidates are inactive: their
Lagrange multipliers are zero and the complementary-slackness binaries
$\delta^{q,\mathrm{dist}}_{i,t}$ and $\delta^{q,\mathrm{ori}}_{j,t}$
\cite{ral20} remain zero across all demonstrations $q$ and timesteps $t$.
To simplify downstream planning, we retain a distance candidate $i$ (or
orientation candidate $j$) only if
$\delta^{q,\mathrm{dist}}_{i,t}=1$ (respectively,
$\delta^{q,\mathrm{ori}}_{j,t}=1$) for at least one $(q,t)$, indicating that
the constraint is tight along at least one demonstrated trajectory.}

A binary being active in the MILP solution does not by itself certify that
the corresponding constraint is physically meaningful --- it may arise from
numerical degeneracy or be trivially satisfied everywhere.
We therefore apply a sensitivity check to each retained candidate.
\logicfix{For a candidate with learned parameter $\theta^\star$, we re-solve the relaxed MILP~\eqref{eq:keypoint_inverse_problem_relaxed} after perturbing that parameter by a small $\delta>0$ in the constraint-loosening direction. For the distance and orientation parameterizations in \eqref{eq:keypoint_mask_distance_constraint}--\eqref{eq:keypoint_orientation_constraint}, this corresponds to imposing $\rho_i\le\rho_i^\star-\delta$ or $\sigma_j\le\sigma_j^\star-\delta$, respectively. We record the resulting change in stationarity residual $\Delta S$. A large increase ($\Delta S\gg0$) indicates that the candidate is required to explain the demonstration's local optimality (see \cite{ral20}); if $\Delta S\approx0$, it contributes negligibly and is discarded.}
\logicfix{The most representative candidates that pass both the activity and sensitivity checks are retained for downstream planning.}

\section{Additional Results}
\looseness-1\noindent\textit{Ground-truth CIOC baseline.\quad}
To provide a hard-constraint
\begin{wrapfigure}{r}{0.4\linewidth}
    \centering
    \vspace{-12pt}
    \includegraphics[width=\linewidth]{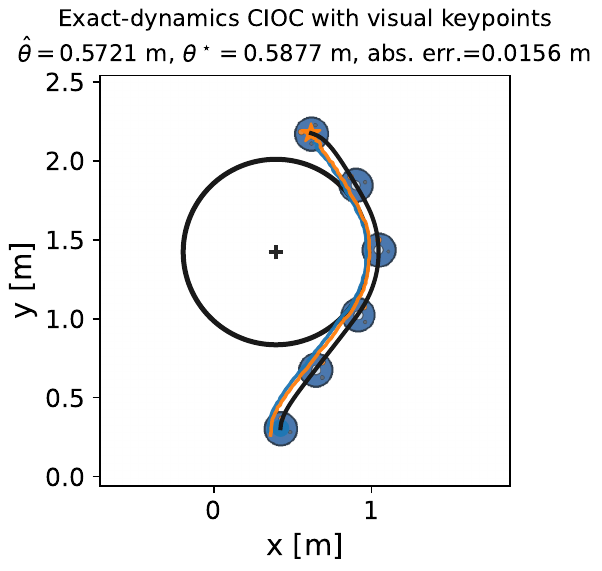}
    \vspace{-15pt}
    \caption{Ground-truth CIOC demonstration.}
    \label{fig:CIOC}
    \vspace{-10pt}
\end{wrapfigure}
CIOC upper bound while retaining the visual representation used by KCL, we rendered an optimal Dubins-car demonstration and extracted keypoints using the same visual keypoint extraction pipeline. We retain the exact analytic Dubins dynamics and assume that the tracked robot keypoints are rigidly attached to the vehicle. The relaxed CIOC solve recovers a clearance of $\hat{\theta}=0.5721\,\mathrm{m}$ from the ground-truth $\theta^\star=0.5877\,\mathrm{m}$, corresponding to an absolute error of $0.0156\,\mathrm{m}$ ($2.66\%$). The total relaxed stationarity residual is $1.18\times10^{-2}$. This result shows that a hard-constrained CIOC formulation with exact dynamics still recovers the constraint accurately when operating on visually extracted keypoints.

\noindent\textit{Nonconvex Constraints.\quad}
While built for distance/
\begin{wrapfigure}{r}{0.4\linewidth}
    \centering
    \vspace{-12pt}
    \includegraphics[width=\linewidth]{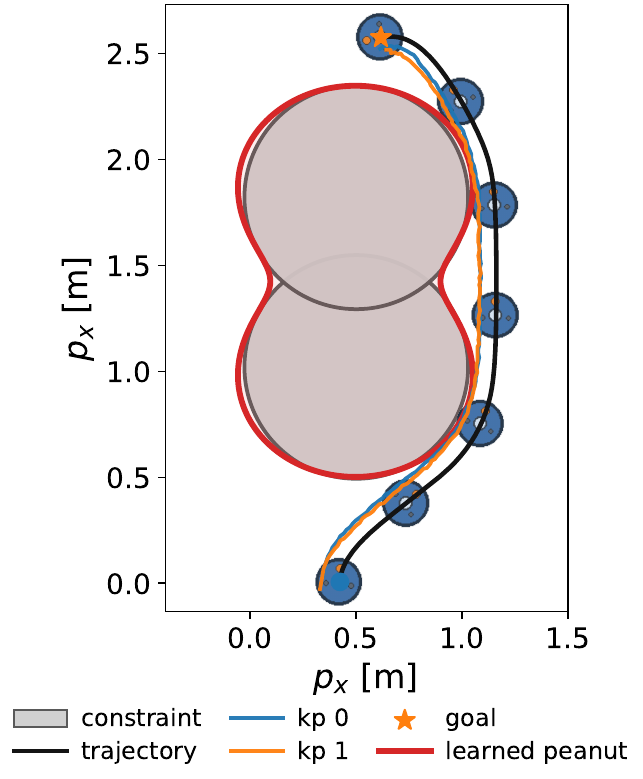}
    \vspace{-15pt}
    \caption{Peanut constraint with Dubins car dynamics.}
    \label{fig:peanut}
    \vspace{-10pt}
\end{wrapfigure}
 \looseness-1  orientation constraints, KCL can learn nonconvex safe or unsafe sets, e.g., we learn a red non-convex unsafe set 
$g(P_t;a,b,R)=$ $ -[ (p_{x,t}^2+p_{y,t}^2-R^2)^2 + aR^2 p_{x,t}^2 - bR^4 ]$ $\leq 0$ to approximate the union of two true circle obstacles. Given one locally-optimal demo avoiding the circles under Dubins dynamics, KCL learns the parameters $\hat a = 2.083, \hat b = 1.1 , \hat{R}=0.645\,\mathrm{m}$ \logicfix{with false-positive and false-negative rates of} $1.11\%$ and $0.20\%$, respectively, relative to the true constraint evaluated over $(p_x,p_y) \in [-1,1.5]\times[0,2.5]$. Similarly, if the true constraint were $g$ and its parameterization unknown to the learner, KCL could instead learn it approximately as a union of two circular distance constraints.

\begin{figure}[t]
        \centering
        \includegraphics[width=\linewidth]{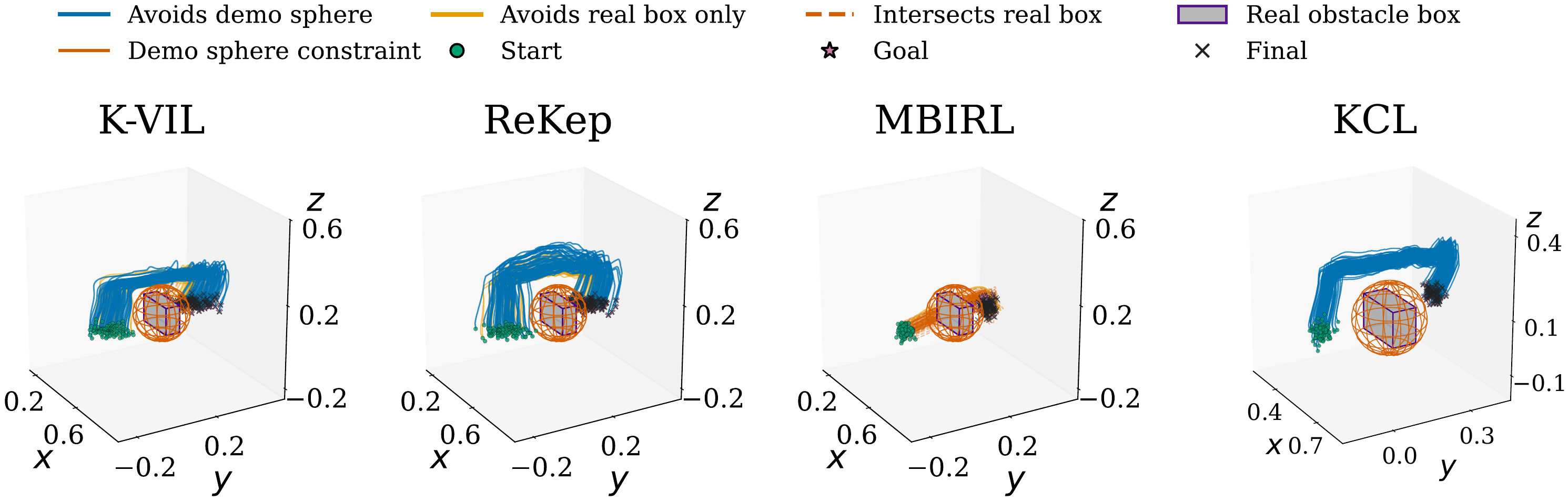}
        \caption{Manipulation baseline comparison. Planning results. For navigation and manipulation planning, blue trajectories satisfy the demonstration-inferred clearance constraint, yellow trajectories avoid the ground-truth obstacle but violate the demonstration-inferred clearance, and dashed orange trajectories collide with the ground-truth obstacle.
    (a) Navigation baseline comparison. 
    (b) Manipulation baseline comparison.}
        \label{fig:app_manipulation_baseline}
    \end{figure}

\begin{figure}
    \centering
    \includegraphics[width=0.98\linewidth]{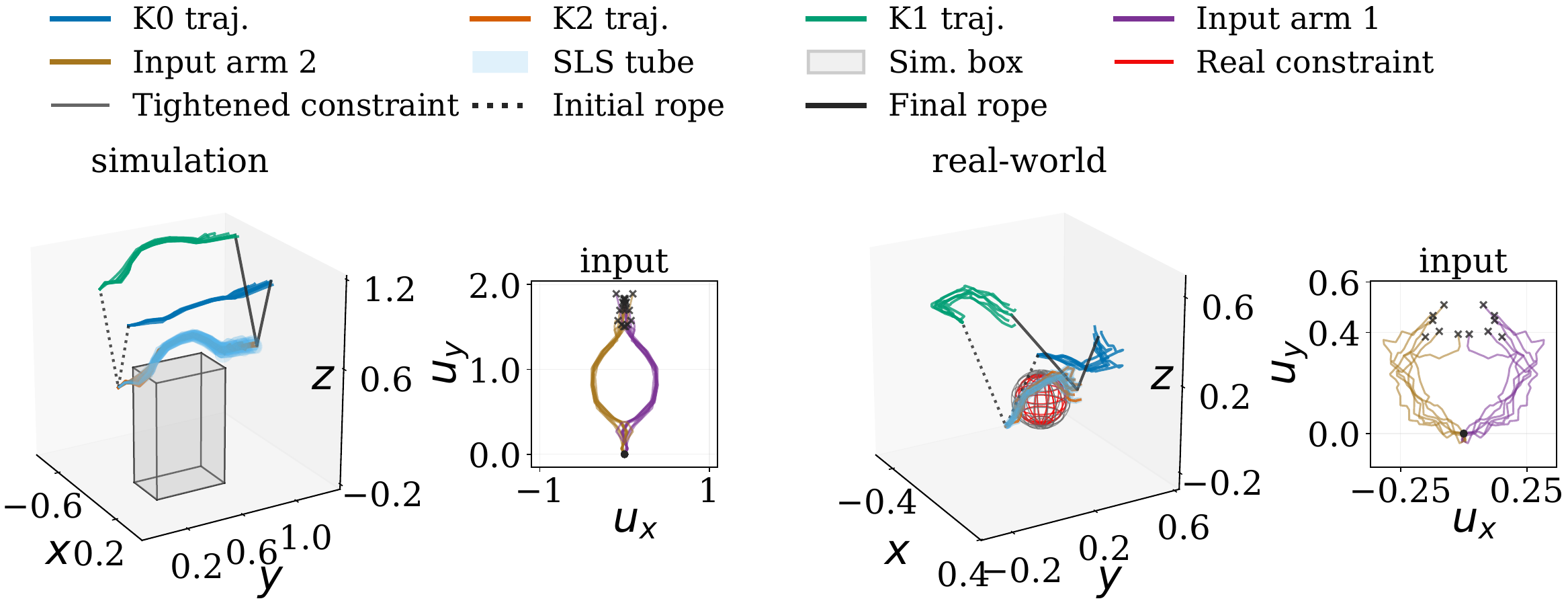}
    \caption{Rope planning with the learned keypoint constraint. Three trajectories denote rope keypoint trajectories; the light-blue tube is the SLS safety tube around the constrained keypoint.}
    \label{fig:app_rope_planning}
\end{figure}

% \vspace{-5pt}
\subsection{Baselines and Planning}

\logicfix{The baseline comparison and rope-planning visualization are shown in Figs.~\ref{fig:app_manipulation_baseline} and~\ref{fig:app_rope_planning}, respectively.} The curves in the navigation and manipulation panels show the centroid trajectory of the manipulated object for visualization clarity. Therefore, some trajectories that appear visually separated from the obstacle can still be unsafe: one or more object keypoints may enter the obstacle box or the demonstration constraint even when the centroid remains outside. In the rope planning real panel, the light-blue tube around the constrained rope keypoint is the SLS safety tube, which bounds the possible tracking deviation induced by accumulated keypoint-dynamics error. The gray wireframe sphere is the tightened constraint, obtained by enlarging the nominal clearance radius by the SLS tube radius. Hence, satisfying the tightened constraint means that the nominal plan leaves enough margin so that the disturbed closed-loop trajectory remains outside the original demonstration-followed constraint under the calibrated dynamics-error bound.

\end{document}